\documentclass[11pt]{article}
\usepackage[margin=1in]{geometry}
\usepackage{authblk}

\usepackage[utf8]{inputenc}
\usepackage[T1]{fontenc}
\usepackage{hyperref}
\usepackage{url}
\usepackage{booktabs}
\usepackage{amsfonts}
\usepackage{nicefrac}
\usepackage{microtype}
\usepackage{xcolor}

\usepackage{graphicx}
\usepackage{algorithm}
\usepackage{algorithmic}
\usepackage{epstopdf}

\usepackage{amsmath}

\usepackage{amssymb}
\usepackage{mathtools}
\usepackage{amsthm}

\definecolor{algcomment}{RGB}{100,100,160}

\usepackage{enumitem}
\usepackage{multirow}
\usepackage{wrapfig}

\usepackage{natbib}

\usepackage[capitalize,noabbrev]{cleveref}

\theoremstyle{plain}
\newtheorem{theorem}{Theorem}[section]
\newtheorem{proposition}[theorem]{Proposition}

\newtheorem{corollary}[theorem]{Corollary}
\theoremstyle{definition}

\theoremstyle{remark}
\newtheorem{remark}[theorem]{Remark}

\title{Discrete Flow Matching for Offline-to-Online \\Reinforcement Learning}

\author{Fairoz Nower Khan}
\author{Nabuat Zaman Nahim}
\author{Peizhong~Ju}

\affil{Department of Computer Science, University of Kentucky\\ 
\mbox{Lexington, Kentucky 40506, USA}}
\affil[ ]{fairoz.khan@uky.edu, nabuat.nahim@uky.edu, peizhong.ju@uky.edu}
\date{}

\begin{document}

\maketitle

\begin{abstract}
 
Many reinforcement learning (RL) tasks have discrete action spaces, but most generative policy methods based on diffusion and flow matching are designed for continuous control. Meanwhile, generative policies usually rely heavily on offline datasets and offline-to online RL is itself challenging, as the policy must improve from new interaction without losing useful behavior learned from static data. To address those challenges, we introduce DRIFT, an online fine-tuning method that updates an offline pretrained continuous time Markov chain (CTMC) policy with an advantage-weighted discrete flow matching loss. To preserve useful pretrained knowledge, we add a path-space penalty that regularizes the full CTMC trajectory distribution, rather than only the final action distribution. For large discrete action spaces, we introduce a candidate-set approximation that updates the actor over a small subset of actions sampled from reference-policy rollouts and uniform exploration. Our theoretical analysis shows that focusing on a small set of high probability actions can reduce computation while keeping the approximation error small, making our method practical for large action spaces. We evaluate on macro-action MinAtar games where offline-pretrained DRIFT outperforms its online-only counterpart across all five games and achieves the highest mean return on three and on oracle free Jericho, DRIFT achieves a mean normalized score of 7.3\% with a lightweight GRU encoder, outperforming the evaluated pretrained language-model baselines. Ablations demonstrate the importance of appropriately weighted path-space regularization and reference-guided candidate construction.

 
\end{abstract}

\section{Introduction}
\label{sec:intro}
 
Offline-to-online reinforcement learning (RL) asks a simple but important question: how can a policy learned from static data keep improving through real interaction? 
Offline RL trains agents from logged data without costly or risky exploration~\citep{levine2020offline,kumar2020conservative,kostrikov2022offline}, but the learned policy is limited by the coverage and quality of that dataset. 
If good actions or important states are missing from the data, offline training alone cannot discover them. 
Offline-to-online fine-tuning addresses this limitation by using the offline policy as a strong initialization and then improving it with online experience~\citep{nair2020awac,lee2022offline,nakamoto2023cal,zhang2023policy}. 
The challenge is to improve the policy without forgetting the useful behavior structure learned during offline pre-training.
 
Diffusion~\citep{ho2020denoising,song2021score} and flow-matching~\citep{lipman2023flow,liu2022rectified} policies capture multimodal action distributions in offline data~\citep{wang2023diffusion,chen2023offline,kang2023efficient,zhang2025energy,park2025flow}, 
and can guide exploration during offline-to-online fine-tuning. 
However, these methods rely on continuous dynamics that do not naturally apply to discrete spaces, leaving out domains such as recommendation, combinatorial control, scheduling, and token-level decision making.
Discrete flow matching~\citep{gat2024discrete, campbell2024generative} replaces continuous flows with continuous-time Markov chain (CTMC) generators
making generative policies possible in discrete action spaces~\citep{khan2026flow}. 
However, adapting pretrained CTMC policies to online RL while preserving their learned behavior and scaling to large action spaces remains challenging.
During the online stage, the generator must move toward 
newly discovered high-reward actions, preserve useful pretrained behavior, 
and scale when the action space is large.
Existing offline discrete flow methods do not address online adaptation, while offline-to-online methods are mostly designed for continuous spaces.
 

\paragraph{Our method.}
Online fine-tuning of discrete generative policies faces two major challenges. The policy must discover better actions without forgetting useful behavior learned offline, while evaluating all possible actions becomes prohibitively expensive in large discrete action spaces. We introduce DRIFT to address both challenges. We combine advantage-weighted discrete flow matching with path-space regularization to improve the policy while limiting changes to its pretrained sampling dynamics. To scale to large action spaces, we update the policy using a small candidate set constructed from reference-policy rollouts and uniform exploration. However, excluding actions changes both the target policy and the CTMC transition rates, making it difficult to guarantee that the approximation remains accurate. Our theoretical analysis addresses this problem by showing that the error in the target generator is bounded by the probability mass excluded from the candidate set (Theorem~\ref{thm:generator-stability}). This provides a theoretical basis for reducing computation while controlling approximation error.

\begin{figure}[t]
  \centering
  \includegraphics[width=0.80\linewidth]{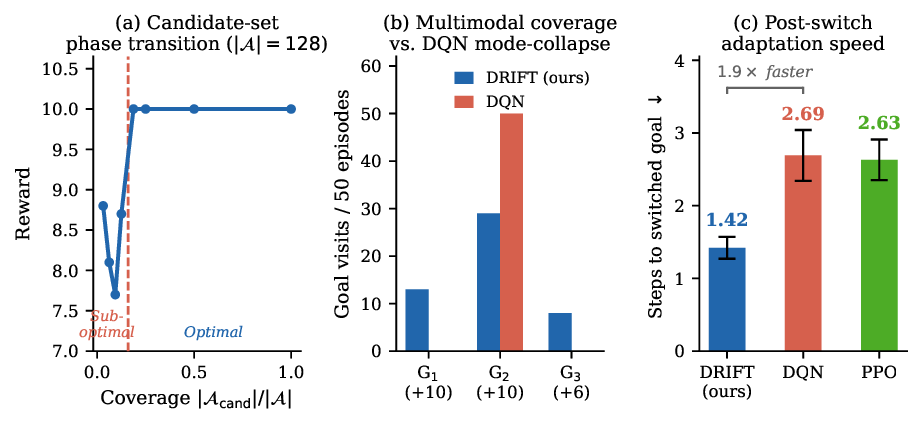}
  \caption{
  Three empirical motivations for our method.
  \textbf{(a)} The candidate-set approximation shows sharp
  performance phase transition at 16\% action coverage.
  \textbf{(b)} Our CTMC policy maintains multimodal coverage, while DQN mode-collapses to single goal.
  \textbf{(c)} In a goal-switch environment, our stochastic policy
  adapts to new reward signal 1.9$\times$ faster than DQN and PPO.
  }
  \label{fig:motivation}
\end{figure}


Our motivating experiments (\cref{fig:motivation}) show our generative policy maintains multimodal goal coverage where
DQN~\citep{mnih2015human} mode-collapses, our stochastic
generator adapts $1.9\times$ faster than DQN and
PPO~\citep{schulman2017proximal} after a reward shift, and the candidate-set approximation reaches optimal
performance at just 16\% action coverage in a 128-action gridworld, matching our coverage bound in Proposition~\ref{prop:coverage}. Together, these observations identify a common regime where standard discrete-action RL struggles: structured action spaces with multimodal actions, shifting rewards, or large action spaces where exhaustive evaluation is infeasible. This work is a first step in that direction which we start by evaluating how our method performs on standard benchmarks, comparing with baselines on Jericho text games, MinAtar, macro-action MinAtar and discretised MuJoCo tasks.

 
 
 
 

\section{Related Work}

\paragraph{Offline-to-online RL and discrete-action RL.}
Offline-to-online methods stabilize the transition from static data to 
live interaction through advantage weighting~\citep{nair2020awac}, 
balanced replay with pessimistic 
ensembles~\citep{lee2022offline}, calibrated value 
learning~\citep{nakamoto2023cal}, policy 
expansion~\citep{zhang2023policy}, and adaptive 
constraints~\citep{li2026state}. These approaches target continuous 
control or use standard policy classes. In discrete-action RL, 
DQN~\citep{mnih2015human}, Double 
DQN~\citep{van2016deep}, Rainbow~\citep{hessel2018rainbow}, 
A2C~\citep{mnih2016asynchronous}, and 
PPO~\citep{schulman2017proximal} provide strong online baselines.
Categorical policies such as PPO can represent multimodal action distributions. Our motivation for CTMC policies is instead their explicit stochastic generation process, which allows regularization of the sampling trajectory.

\paragraph{Generative policies for RL.}
Diffusion and flow-based policies capture multimodal behavior in 
offline data. Diffusion-QL~\citep{wang2023diffusion}, 
IDQL~\citep{hansen2023idql}, efficient diffusion 
policies~\citep{kang2023efficient}, energy-weighted flow 
matching~\citep{zhang2025energy}, and Flow 
Q-Learning~\citep{park2025flow} all improve offline RL through 
expressive action modeling with value guidance. 
Online optimization of continuous action generative policies has been explored through flow based policy optimization~\citep{zhang2025reinflow,mcallister2025flow}. GoRL~\citep{zhang2026generative} takes a complementary approach by optimizing a tractable latent policy and refining a conditional generative decoder on separate timescales. Unlike these continuous control approaches, DRIFT focuses on offline-to-online fine tuning of pretrained discrete action CTMC policies.

\paragraph{Discrete generative policies for RL.}
RL-D$^2$~\citep{ma2026reinforcement} trains discrete diffusion policies for online RL by matching a policy mirror descent target using forward or reverse KL divergence, demonstrating strong performance on macro-action MinAtar and Atari. QDFM~\citep{khan2026flow} uses value-weighted discrete flow matching to learn CTMC policies for offline RL. Inspired by these work, our method addresses offline-to-online adaptation of pretrained CTMC policies. We regularize the full CTMC path during fine-tuning rather than relying only on policy-level KL constraints, while improving online experience using path-space regularization and reference-guided candidate sets to improve pretrained policies in large discrete action spaces.

\paragraph{Discrete generative models and trust-region fine-tuning.}
Discrete diffusion and flow models replace Euclidean dynamics with 
Markov processes over finite state spaces for 
sequences~\citep{gat2024discrete}, molecular 
design~\citep{campbell2024generative}, and biological 
generation~\citep{holderrieth2025generator,wang2025fine}. 
Discrete Adjoint Matching~\citep{so2026discrete} studies reward optimization of CTMC generators, while TR2-D2~\citep{tang2025tr2} develops trajectory aware fine tuning for discrete diffusion models. DRIFT focuses on offline to online RL by combining advantage weighted DFM, path space KL regularization, and reference guided candidate learning. Our path-space 
regularization builds on trust-region policy 
improvement~\citep{peters2010reps,schulman2015trust,schulman2017proximal}and KL-regularized fine-tuning from human 
feedback~\citep{ziegler2019fine,ouyang2022training} requiring careful design considerations which we detail in the next sections.

\section{Preliminaries}
\label{sec:prelims}

\subsection{Markov Decision Process and Offline Data}
\label{sec:mdp}

We consider a Markov decision process (MDP) $\mathcal{M} = (\mathcal{S}, \mathcal{A}, r, P, \gamma)$, where $\mathcal{S}$ is the state space, $\mathcal{A}$ is a \emph{finite} discrete action space with $K = |\mathcal{A}|$ actions, $r\colon \mathcal{S}\times\mathcal{A}\to\mathbb{R}$ is the reward function, $P$ is the transition kernel, and $\gamma\in(0,1)$ is the discount factor.
The state-action value function is $Q^\pi(s,a) = \mathbb{E}_\pi\!\bigl[\sum_{t=0}^\infty \gamma^t r(s_t,a_t) \mid s_0\!=\!s, a_0\!=\!a\bigr]$.

An offline dataset $\mathcal{D}_{\mathrm{off}} = \{(s_i,a_i,r_i,s_i')\}_{i=1}^N$ is collected by an unknown behavior policy.
The KL-regularized optimal policy admits the well-known Boltzmann form~\citep{todorov2006linearly,haarnoja2017reinforcement,lu2023contrastive}:
$\pi^\star(a \mid s)
\;\propto\;
\pi_{\mathrm{ref}}(a \mid s)\,\exp\!\bigl(A(s,a)/\beta\bigr),$
where $A(s,a) = Q(s,a) - V(s)$ is the advantage, $\pi_{\mathrm{ref}}$ is a reference policy, and $\beta>0$ is a temperature that trades off reward maximization against proximity to $\pi_{\mathrm{ref}}$.

\subsection{Discrete Flow Matching via CTMCs}
\label{sec:dfm_prelim}

Discrete Flow Matching (DFM)~\citep{campbell2024generative,gat2024discrete} extends the flow-matching paradigm from continuous Euclidean spaces to finite discrete spaces by replacing continuous vector fields used in standard flow matching with Continuous-Time Markov Chains (CTMCs).

A CTMC over the action space $\mathcal{A}$ is specified by a time-dependent rate matrix (or generator) $u_t(i\!\to\!j \mid s)$, which represents the instantaneous rate of jumping from action $i$ to action $j$ at time $t$, given state $s$. It satisfies the following constraints for every source action $i$:
\begin{equation}
\label{eq:rate_constraints}
u_t(i\!\to\!j \mid s) \ge 0 \;\;\forall\, j\neq i,
\qquad
\sum\nolimits_{j\in\mathcal{A}} u_t(i\!\to\!j \mid s) = 0.
\end{equation}
The first constraint ensures nonnegative off-diagonal transition rates and the second ensures probability conservation.
The marginal probability mass function $p_t$ evolves according to the Kolmogorov forward equation: $\frac{d}{dt}\,p_t(j \mid s)
= \sum\nolimits_{i\in\mathcal{A}} p_t(i \mid s)\,u_t(i\!\to\!j \mid s).$

DFM learns a parametric rate model $u_\theta$ by regressing it against a tractable \emph{conditional} target generator derived from a prescribed probability path connecting a source distribution $p_0$ to a target distribution $p_1$.
At inference time, actions are sampled by simulating the learned CTMC from $t=0$ to $t=1$ using Euler discretization, as explained in Appendix~\ref{app:ctmc_sim}.

\subsection{Offline Pre-Training via Advantage-Weighted DFM}
\label{sec:pretraining_prelim}

Our method requires a pretrained reference generator $u_{\mathrm{ref}}$ but does not depend on a particular pretraining procedure. In our experiments, we obtain it through advantage-weighted DFM, which learns to transport a uniform prior toward a policy based on offline advantage estimates. Other pretrained discrete-action policies can also be distilled into a CTMC generator. Implementation details are provided in Appendix~\ref{app:pretraining}.

\section{Our Method: Online fine-tuning of Discrete Flow Matching Policy}
\label{sec:method}

Offline-pretrained policies are limited by the coverage and quality of their data~\citep{levine2020offline}. Online fine-tuning can discover better actions but risks forgetting useful pretrained behavior. We introduce DRIFT, which improves a pretrained CTMC generator using online rewards 
while discouraging deviations from its learned behavior through path space KL regularization.


\begin{wrapfigure}[33]{r}{0.58\linewidth}
\centering
\small
\refstepcounter{algorithm}
\noindent\textbf{Algorithm~\thealgorithm} DRIFT: Discrete flow matching with Regularised Information-theoretic Fine-Tuning (sketched version, full algorithm in Appendix~\ref{app:full_algo})
\label{alg:sketch}
\vspace{1pt}
\hrule
\vspace{2pt}

\begin{algorithmic}[1]
\REQUIRE Pre-trained generator $u_{\mathrm{ref}}$, offline data $\mathcal{D}_{\mathrm{off}}$, temperature $\beta$, KL weight $\alpha$, CTMC sub-steps $M$, candidate sizes $N_{\mathrm{roll}}, N_{\mathrm{rand}}$, refresh interval $K$, flow-time truncation $\delta \in (0,1)$ 
\STATE Initialize $u_\theta \leftarrow u_{\mathrm{ref}}$, critics $Q_{\phi_1},Q_{\phi_2}$, value $V_\psi$, targets $Q_{\phi_k^-},V_{\psi^-}$, buffer $\mathcal{B}$
\FOR{$n = 1$ \TO $N_{\mathrm{steps}}$}
    \STATE \label{line:collect}Simulate CTMC under $u_\theta$ from $X_0\!\sim\!\mathrm{Unif}(\mathcal{A})$ for $M$ sub-steps; execute $a\!=\!X_M$; store $(s,a,r,s',d)$ in $\mathcal{B}$
    \STATE \label{line:mixed}Sample mixed minibatch: $(1\!-\!\rho)B$ from $\mathcal{B}$ and $\rho B$ from $\mathcal{D}_{\mathrm{off}}$
    \STATE \label{line:critic}Update critics: $Q_{\phi_k} \leftarrow \arg\min\, \mathcal{L}_{Q_k}$ 
    \STATE \label{line:value}Update value: $V_\psi \leftarrow \arg\min\, \mathcal{L}_V$ 
    using lagged $\tilde{\pi}^-$ (Eq.~\ref{eq:lagged_target})
    \FOR{each state $s_b$ in actor batch}
        \STATE \label{line:cand}Build $\mathcal{A}_{\mathrm{cand}}(s_b)$ via $N_{\mathrm{roll}}$ rollouts of frozen $u_{\mathrm{ref}}$ $+$ $N_{\mathrm{rand}}$ uniform actions
        \STATE \label{line:pi_ref}Compute smoothed $\pi_{\mathrm{ref}}$ (Eq.~\ref{eq:pi_ref}) and clipped advantage $\bar{A}$ 
        \STATE \label{line:target_pi}Set target policy $\tilde{\pi}(a \mid s_b) \propto \pi_{\mathrm{ref}}(a \mid s_b)\exp(\bar{A}(s_b,a)/\beta)$ (Eq.~\ref{eq:target_policy})
        \STATE \label{line:target_gen}Construct bridge $p_t$ (Eq.~\ref{eq:linear-bridge}) and target generator $u_t^*$ by independent coupling (Eq.~\ref{eq:coupling})
        \STATE \label{line:dfm_loss}
        Compute $\ell_{\mathrm{DFM}}(s_b)$ by sampling
        $t \sim \mathrm{Unif}(0,1-\delta)$ and $i \sim p_t$
        (Eq.~\ref{eq:dfm_loss}).
        \STATE \label{line:kl}Estimate path-space $\widehat{\mathrm{KL}}(s_b)$ against frozen $u_{\mathrm{ref}}$ (Eq.~\ref{eq:path_kl})
    \ENDFOR
    \STATE \label{line:actor_update}Update $u_\theta$ by minimizing $\mathcal{L}_{\mathrm{actor}} = \mathcal{L}_{\mathrm{DFM}} + \alpha\,\widehat{\mathcal{L}}_{\mathrm{KL}}$ (Eq.~\ref{eq:combined_loss})
    \STATE Soft-update targets: $\phi_k^- \!\leftarrow\! (1\!-\!\tau)\phi_k^-\!+\!\tau\phi_k$, $\psi^- \!\leftarrow\! (1\!-\!\tau)\psi^-\!+\!\tau\psi$
    \STATE \label{line:refresh}If $n \bmod K = 0$: set $u_{\mathrm{ref}} \leftarrow u_\theta$
\ENDFOR
\RETURN Fine-tuned generator $u_\theta$
\end{algorithmic}

\vspace{2pt}
\hrule
\end{wrapfigure}

\paragraph{Overview.}
Starting from a reference generator $u_{\mathrm{ref}}$ obtained via advantage-weighted DFM pre-training (Section~\ref{sec:pretraining_prelim}), DRIFT performs iterative online learning.
Each iteration consists of four steps, summarized in Algorithm~\ref{alg:sketch}:
(1)~collect online transitions by rolling out the current CTMC generator in the environment (Line~\ref{line:collect}),
(2)~update critics and a value network from a mixed replay buffer (Lines~\ref{line:critic}--\ref{line:value}),
(3)~construct a reward-weighted target policy and its corresponding DFM target generator over a candidate action set (Lines~\ref{line:cand}--\ref{line:target_gen}),
(4)~update the generator by minimizing a combined DFM loss plus path-space KL penalty (Lines~\ref{line:dfm_loss}--\ref{line:actor_update}).
The reference generator is periodically refreshed to track improving performance (Line~\ref{line:refresh}). 
On Freeway, refreshing the reference every 500 actor updates improves returns compared with keeping it fixed or updating it at every step (Table~\ref{tab:ref_refresh_ablation}).


\subsection{Environment Interaction and Replay}
\label{sec:interaction}

Actions are generated by simulating the current CTMC generator $u_\theta$ forward in flow time.
Starting from $X_0 \sim \mathrm{Uniform}(\mathcal{A})$, we run $M$ Euler sub-steps with step size $\Delta t = 1/M$ (Algorithm~\ref{alg:ctmc_sim} in Appendix~\ref{app:ctmc_sim}).
At each sub-step, the process either stays at the current action or jumps to a new one with probability proportional to the learned transition rates.
The terminal action $a = X_M$ is executed in the environment, and the resulting transition $(s,a,r,s',d)$ is stored in a replay buffer $\mathcal{B}$. To prevent the fine-tuning distribution from drifting too far from the pre-trained policy, each training minibatch mixes online transitions from $\mathcal{B}$ with a fraction $\rho$ of offline transitions from $\mathcal{D}_{\mathrm{off}}$ (Line~\ref{line:mixed}), following offline-to-online practice~\citep{nakamoto2023cal,lee2022offline}.

\subsection{Candidate Set Construction and Target Policy}
\label{sec:target_policy}

Evaluating every action during online fine-tuning becomes prohibitively expensive in large discrete action spaces.  To address this, we carefully construct a small candidate set $\mathcal{A}_{\mathrm{cand}}(s)$ using the replay action, reference generator rollouts, and uniform sampling. First, $N_{\mathrm{roll}}$ rollouts of the frozen reference generator provide actions favored by the pretrained policy, preserving useful offline knowledge. However, relying only on these actions may prevent the policy from discovering better actions missing from the offline data. After removing duplicates, we use uniformly sampled actions to fill the remaining slots up to budget $B$, maintaining candidate diversity and including
actions not proposed by the reference generator. This combination focuses computation on promising actions while allowing the policy to discover new ones.

\paragraph{Smoothed reference policy.}
The reference distribution over the candidate set is estimated from rollout frequencies with additive smoothing (Line~\ref{line:pi_ref}):
\begin{equation}
\label{eq:pi_ref}
\pi_{\mathrm{ref}}(a \mid s)
\;\propto\;
\frac{\mathrm{count}(a)}{N_{\mathrm{roll}}} + \frac{\epsilon}{|\mathcal{A}_{\mathrm{cand}}(s)|},
\qquad a \in \mathcal{A}_{\mathrm{cand}}(s).
\end{equation}

Here, the first term, $\mathrm{count}(a)/N_{\mathrm{roll}}$, estimates how frequently the pretrained policy selects action $a$, giving greater weight to actions favored by the reference generator. The second term, $\epsilon/|\mathcal{A}_{\mathrm{cand}}(s)|$, adds a small positive weight to every candidate, including uniformly sampled actions that never appeared in the reference rollouts. Without this smoothing, such actions would receive zero probability and could not gain probability through the subsequent advantage-weighted update. The resulting weights are normalized over the candidate set to obtain a valid probability distribution.

\paragraph{Target policy.}
The target policy $\tilde{\pi}$ is the solution to the KL-regularized policy improvement problem~\citep{ziegler2019fine,peters2010relative} (Line~\ref{line:target_pi}):
\begin{equation}
\label{eq:target_policy}
\tilde{\pi}(a \mid s)
= \frac{
\pi_{\mathrm{ref}}(a \mid s)\,\exp\!\bigl(\bar{A}(s,a)/\beta\bigr)
}{
\sum_{a'\in\mathcal{A}_{\mathrm{cand}}(s)} \pi_{\mathrm{ref}}(a' \mid s)\,\exp\!\bigl(\bar{A}(s,a')/\beta\bigr)
}.
\end{equation}

Here, $\bar{A}(s,a)$ is the normalized and clipped conservative advantage defined in Section~\ref{sec:critic}.
We extend this target to the full action space by setting $\tilde\pi(a\mid s)=0,  a\notin \mathcal A_{\mathrm{cand}}(s).$ 
The advantage term increases the probability of actions with higher estimated values, while $\beta$ controls how strongly the target policy moves away from the reference distribution. The key question is whether restricting this update to a small candidate set introduces substantial error. Our analysis shows that this error depends on the target probability mass excluded from the candidate set (Proposition~\ref{prop:coverage}). Under the assumptions of Theorem~\ref{thm:generator-stability}, the resulting CTMC generator error is also bounded by this missing mass. Thus, the candidate set can remain small when it captures most of the target probability mass.

\subsection{Critic and Value Network Updates}
\label{sec:critic}

We maintain two independent critic networks $Q_{\phi_1}, Q_{\phi_2}$ and a value network $V_\psi$. 
Using two critics helps reduce overestimation of action values, so taking their minimum can prevent the policy from favoring actions with inaccurately high estimated rewards. We update the critics using mixed offline and online replay, while a lagged target policy constructed from frozen networks provides stable supervision for the value network.

For the actor update, we compute the conservative advantage
\[
A(s,a) = \min_{k=1,2} Q_{\phi_k}(s,a) - V_{\psi^-}(s).
\]
We normalize the advantages over the candidate set and clip them to obtain
\[
\bar{A}(s,a)
= \mathrm{clip}\left(
\frac{A(s,a)-\mu_s}{\sigma_s+\varepsilon},
-c,c
\right),
\]
where $\mu_s$ and $\sigma_s$ are the mean and standard deviation of the candidate advantages. We use $c=3$ for macro MinAtar, with sensitivity results in Appendix~\ref{app:macro_ablations}. The complete critic and value update equations are provided in Appendix~\ref{app:critic_value}.

\subsection{Actor Update: DFM Loss and Path-Space KL}
\label{sec:actor}

The actor update has two components: a discrete flow matching loss that steers the generator toward $\tilde{\pi}$, and a path-space KL penalty that prevents catastrophic forgetting.

\paragraph{Linear bridge and target generator.}

We define a linear interpolation from the full-action uniform prior 
$p_0 = \mathrm{Uniform}(\mathcal{A})$ to the target policy:
\begin{equation}
p_t(a \mid s) = (1 - t)\, p_0(a) + t\, \tilde{\pi}(a \mid s),
\qquad
\dot{p}_t(a \mid s) = \tilde{\pi}(a \mid s) - p_0(a).
\label{eq:linear-bridge}
\end{equation}
For the candidate-set update, we view $\tilde\pi(\cdot\mid s)$ as a distribution on the full action space by setting $\tilde\pi(a\mid s)=0$ for $a\notin\mathcal A_{\mathrm{cand}}(s)$. Because $\tilde{\pi}$ spreads mass across multiple actions, the Kolmogorov forward equation does not uniquely determine a generator consistent with $\dot{p}_t$.
We resolve this via the \emph{independent coupling transport}~\citep{campbell2024generative,holderrieth2025generator}, which constructs a valid target generator $u_t^*$ by routing outflow from each losing action to all gaining actions proportionally to their demand (Line~\ref{line:target_gen}).
Concretely, for $i \neq j$:
\begin{equation}
\label{eq:coupling}
u_t^*(i\!\to\!j \mid s)
= \frac{(\dot{p}_t(i))^-\,(\dot{p}_t(j))^+}{p_t(i \mid s)\,Z_t},
\qquad
Z_t = \sum\nolimits_{a} (\dot{p}_t(a))^+.
\end{equation}
The validity of this generator and its mass conservation proof are provided in Appendix~\ref{app:coupling}.

\paragraph{Source actions outside the candidate set.}
Since the source distribution covers the full action space, intermediate actions may lie outside the candidate set. The independent coupling treats these actions as sources of probability mass and transfers their mass to candidate actions with positive demand. Further details are provided in Appendix~\ref{app:coupling}.


\paragraph{Discrete flow matching loss.}
For each state $s_b$, we sample a flow time 
$t \sim \mathrm{Uniform}(0, 1 - \delta)$ and an intermediate action $i \sim p_t(\cdot \mid s_b)$, then regress the learned rates against the target rates (Line~\ref{line:dfm_loss}):
\begin{equation}
\label{eq:dfm_loss}
\mathcal L_{\mathrm{DFM}}(\theta)
=
\frac{1}{B}
\sum_{b=1}^{B}
\sum_{j\neq i_b}
\left(
u_\theta(i_b\!\to\! j,t_b\mid s_b)
-
u^*_{t_b}(i_b\!\to\! j\mid s_b)
\right)^2 .
\end{equation}

\paragraph{Path-space KL regularization.} Unlike the policy KL commonly used in RL, which compares two action distributions at a state, our KL compares two CTMC path distributions, so it constrains how probability mass moves over flow time.
We penalize deviations across the \emph{entire trajectory} of the CTMC, including both the jump destinations and the holding times (Line~\ref{line:kl}).
For CTMCs, the KL divergence between path measures $\mathcal{P}_{u_\theta}$ and $\mathcal{P}_{u_{\mathrm{ref}}}$ admits a tractable decomposition via the Radon--Nikodym derivative~\citep{kipnis2013scaling,zhang2025reinflow}:

\begin{equation}
\label{eq:path_kl}
\begin{split}
\mathrm{KL}\!\bigl(\mathcal{P}_{u_\theta}\!\parallel\!\mathcal{P}_{u_{\mathrm{ref}}}\bigr)
= \mathbb{E}_{\mathcal{P}_{u_\theta}}\!\!\Biggl[
&\sum_{k:\mathrm{jumps}}
\log\frac{u_\theta(X_{t_k^-}\!\!\to\!X_{t_k},t_k \mid s)}{u_{\mathrm{ref}}(X_{t_k^-}\!\!\to\!X_{t_k},t_k \mid s)} \\
&+ \int_0^1 \!\bigl(\lambda_{\mathrm{ref}}(X_t,t \mid s) - \lambda_\theta(X_t,t \mid s)\bigr)dt
\Biggr],
\end{split}
\end{equation}
where $\lambda_\theta(i,t \mid s) = \sum_{j\neq i} u_\theta(i\!\to\!j,t \mid s)$ is the total exit rate.
We use the equivalent compensator representation of this KL. This is a Monte Carlo plug-in surrogate where gradients flow only through the 
evaluated $u_\theta$ and $\lambda_\theta$ along the simulated path, so the 
term acts as a practical regularizer toward $u_{\rm ref}$ rather than as an 
unbiased gradient estimator of the path-space KL. For each sampled state $s_b$, we estimate this quantity with a Monte Carlo plug-in estimator
$\widehat{\mathcal L}_{\mathrm{KL}}(s_b, \theta)$ along a simulated CTMC path. 
The compensator identity is exact, but differentiating its estimate with detached CTMC states gives a biased gradient. Further details and gradient analysis are provided in Appendix~\ref{app:path_kl}.

\paragraph{Combined fine-tuning objective.}
The generator is then updated by minimizing
\begin{equation}
\label{eq:combined_loss}   
\mathcal L_{\mathrm{actor}}(\theta)
=
\mathcal L_{\mathrm{DFM}}(\theta)
+
\alpha \widehat{\mathcal L}_{\mathrm{KL}}(\theta),
\end{equation}
where $\widehat{\mathcal L}_{\mathrm{KL}}(\theta)
=
\frac{1}{B}\sum_{b=1}^B
\widehat{\mathcal L}_{\mathrm{KL}}(s_b,\theta)$ and $\alpha>0$ controls the strength of the KL penalty, which encourages similarity to the reference without guaranteeing a strict bound on path divergence.

\subsection{Rate Network Parameterization}
\label{sec:param}

The CTMC generator $u_\theta$ constraints (Eq.~\ref{eq:rate_constraints}) are enforced by construction.
The network outputs unconstrained logits $g_\theta(i\!\to\!j, t \mid s)$ for $j \neq i$, which are passed through a softplus to obtain nonnegative off-diagonal rates:
$u_\theta(i\!\to\!j,t \mid s) = \mathrm{softplus}(g_\theta(i\!\to\!j,t \mid s))$.
The diagonal is set to $u_\theta(i\!\to\!i,t \mid s) = -\sum_{j\neq i} u_\theta(i\!\to\!j,t \mid s)$,
guaranteeing a valid generator at every parameter setting.
Our formulation is independent of the network architecture, provided that the output rates satisfy the CTMC generator constraints.

\subsection{Theoretical Analysis}
\label{sec:theory}

We analyze the error caused by replacing the full action space with a 
candidate set. For the analysis, let $\tilde\pi(\cdot \mid s)$ denote 
the ideal target policy on the full action space, and define its 
candidate-restricted renormalization as
\[
\tilde\pi_{\mathrm{cand}}(a \mid s) 
= \frac{\tilde\pi(a \mid s)\, \mathbf{1}\{a \in \mathcal{A}_{\mathrm{cand}}(s)\}}
       {\sum_{a' \in \mathcal{A}_{\mathrm{cand}}(s)} \tilde\pi(a' \mid s)}.
\]

\begin{proposition}[Coverage Error]
\label{prop:coverage}
Let $\epsilon(s) = \sum_{a \notin \mathcal{A}_{\mathrm{cand}}(s)} \tilde\pi(a \mid s)$
be the target mass excluded by the candidate set. If $\epsilon(s) < 1$, then
\[
\|\tilde\pi_{\mathrm{cand}} - \tilde\pi\|_1 = 2\epsilon(s).
\]
For the expected error, assume that the ideal target and full action reference distribution are fixed, and that the candidate set contains $N_{\mathrm{roll}}$ independent samples from $\pi_{\mathrm{ref}}$ and $N_{\mathrm{rand}}$ independent uniform samples. Then
\[
\mathbb{E}[\|\tilde\pi_{\mathrm{cand}} - \tilde\pi\|_1] 
= 2 \sum_{a \in \mathcal{A}} \tilde\pi(a \mid s)\,
  (1 - \pi_{\mathrm{ref}}(a \mid s))^{N_{\mathrm{roll}}}\,
   (1 - 1/|\mathcal{A}|)^{N_{\mathrm{rand}}}.
\]

\end{proposition}
This expectation is exact under the stated sampling assumptions. Our implementation also includes the replay action and fills unused slots after removing duplicates, so we measure its candidate coverage directly.

Proposition~\ref{prop:coverage} shows that the expected approximation error decreases exponentially as the number of reference rollouts and uniform samples increases. This explains why a relatively small candidate set can accurately approximate the target policy when it captures most of its probability mass. The proof is provided in
Appendix~\ref{app:proof-coverage}.


\begin{theorem}[Generator Stability]
\label{thm:generator-stability}
Let $u^*_t$ and $u^{*,\mathrm{cand}}_t$ be the target generators induced
by $\tilde\pi$ and $\tilde\pi_{\mathrm{cand}}$, respectively. Assume both
bridges use the same source distribution $p_0=\mathrm{Unif}(\mathcal A)$.
Fix $\delta\in(0,1)$ and suppose that for all $t\in[0,1-\delta]$,
$p_t(i),\;p^{\mathrm{cand}}_t(i)\ge p>0, Z_t,\;Z^{\mathrm{cand}}_t\ge Z>0.$
Then, for every source action $i$,
\[
\sum_{j\neq i}
\left|
u^{*,\mathrm{cand}}_t(i\to j)
-
u^*_t(i\to j)
\right|
\le
\frac{C}{p^2 Z^2}
\left\|
\tilde\pi_{\mathrm{cand}}(\cdot\mid s)
-
\tilde\pi(\cdot\mid s)
\right\|_1 ,
\]
for a universal constant $C$. 
\end{theorem}

Theorem~\ref{thm:generator-stability} shows that as the candidate set captures more target-policy probability mass, the resulting CTMC generator approaches the full-action generator. The proof is provided in Appendix~\ref{app:proof-stability}. 
This conditional bound can be loose when target mass coverage is low and does not guarantee strong returns. We examine the practical approximation error in Appendix~\ref{app:macro_details}.


\section{Experiments}
\label{sec:experiments}

We evaluate three aspects of DRIFT: its ability to adapt
pretrained policies across different structured discrete
action spaces, the contribution of path-space regularization
and reference-guided candidate selection, and its practical
scalability as the action space grows. Macro-action
MinAtar (Section~\ref{sec:macro}) tests combinatorial
action spaces ranging from 1,296 to approximately
1.68 million actions, while oracle-free Jericho
(Section~\ref{sec:jericho-noracle}) tests language-based
action generation without access to admissible-action
lists. Alongside returns and learning curves, we report
controlled ablations, candidate coverage, and
wall-clock measurements. Additional benchmarks and
implementation details are provided in the appendix.

\subsection{Macro-action MinAtar}
\label{sec:macro}

\paragraph{Setup.}
A macro executes $K$ primitive actions without observing intermediate frames, so $\mathcal{A}=\{0,\dots,5\}^K$, and all methods use semi-Markov targets. Every offline-to-online method starts from the same 100K-transition dataset per game and macro length, collected by a behaviour double DQN, and then interacts for 1.2M primitive frames, so all methods receive equal experience whatever their macro length. We evaluate every 40K frames with 20 episodes over five seeds and report returns.

\paragraph{Baselines.}
The offline-to-online baselines are RL-D$^2$-O2O~\citep{ma2026reinforcement}, Q-chunking~\citep{li2026reinforcement}, QDFM~\citep{khan2026flow}, BraVE~\citep{landers2026brave}, Cal-QL~\citep{nakamoto2023cal}, IQL~\citep{kostrikov2022offline} and SPA~\citep{li2026state}. The online-only baselines are DRIFT without offline data, RL-D$^2$ as published, DQN-Macro with one output per macro, IMPALA-Macro~\citep{espeholt2018impala} and MASP~\citep{hosu2026meta}, the last at $K{=}4$ only because its similarity matrix has $|\mathcal{A}|^2$ entries. DRIFT uses $B{=}64$, $\beta{=}0.5$, $M{=}10$ and $\alpha{=}0.01$.

\begin{table}[h]
\centering
\caption{Macro-action MinAtar at $K{=}4$ ($|\mathcal{A}|{=}1{,}296$). Methods with an offline phase show the score after pre-training $\to$ the final score after 1.2M online primitive frames. Mean $\pm$ std over five seeds. Bold marks the best macro-action score per game. Primitive-action references do not depend on $K$ and are reported in Table~\ref{tab:macro_k8}.}
\label{tab:macro_k4}
\resizebox{\textwidth}{!}{%
\begin{tabular}{lccccc}
\toprule
Method & Asterix & Breakout & Freeway & Seaquest & Space Invaders \\
\midrule
\multicolumn{6}{l}{\textit{Macro-action, offline-to-online}} \\
DRIFT (ours) & 0.4 $\to$ \textbf{2.1$\pm$0.3} & 0.5 $\to$ 15.4$\pm$3.6 & 11.3 $\to$ \textbf{38.1$\pm$0.5} & 0.6 $\to$ \textbf{7.0$\pm$0.2} & 4.2 $\to$ 36.0$\pm$1.7 \\
RL-D$^2$-O2O      & 0.4 $\to$ 1.2$\pm$0.1 & 0.9 $\to$ 16.4$\pm$0.2 & 9.4 $\to$ 32.8$\pm$2.3  & 0.3 $\to$ 5.9$\pm$0.1 & 4.2 $\to$ \textbf{40.2$\pm$2.8} \\
Q-chunking        & 0.4 $\to$ 0.9$\pm$0.3 & 0.9 $\to$ \textbf{20.4$\pm$0.7} & 25.4 $\to$ 25.1$\pm$0.2 & 0.8 $\to$ 6.9$\pm$0.1 & 4.8 $\to$ 30.4$\pm$1.7 \\
QDFM              & 0.6 $\to$ 1.3$\pm$0.3 & 0.0 $\to$ 14.7$\pm$2.4 & 23.2 $\to$ 25.4$\pm$1.7 & 0.5 $\to$ 1.5$\pm$0.2 & 3.4 $\to$ 24.4$\pm$2.7 \\
BraVE             & 0.5 $\to$ 0.9$\pm$0.0 & 0.0 $\to$ 11.2$\pm$3.2 & 0.0 $\to$ 14.4$\pm$3.5  & 0.3 $\to$ 4.5$\pm$0.2 & 1.4 $\to$ 25.5$\pm$1.8 \\
Cal-QL            & 0.6 $\to$ 0.6$\pm$0.1 & 0.7 $\to$ 5.5$\pm$0.2  & 13.3 $\to$ 3.0$\pm$0.7  & 0.3 $\to$ 5.7$\pm$0.3 & 2.4 $\to$ 18.5$\pm$2.5 \\
IQL               & 0.5 $\to$ 0.7$\pm$0.1 & 0.6 $\to$ 6.9$\pm$2.1  & 0.3 $\to$ 0.2$\pm$0.1   & 0.1 $\to$ 1.2$\pm$0.2 & 4.5 $\to$ 9.0$\pm$0.5 \\
SPA               & 0.8 $\to$ 0.8$\pm$0.0 & 1.2 $\to$ 11.0$\pm$1.6 & 18.9 $\to$ 22.5$\pm$0.8 & 0.6 $\to$ 1.8$\pm$0.1 & 5.6 $\to$ 18.1$\pm$1.0 \\
\midrule
\multicolumn{6}{l}{\textit{Macro-action, online only}} \\
DRIFT, no offline data & 0.8$\pm$0.3 & 3.8$\pm$0.8 & 19.9$\pm$4.7 & 4.0$\pm$0.1 & 26.7$\pm$4.1 \\
RL-D$^2$      & 1.0$\pm$0.2 & 8.7$\pm$1.1 & 11.5$\pm$11.5 & 3.5$\pm$0.5 & 31.7$\pm$1.1 \\
DQN-Macro     & 0.8$\pm$0.0 & 7.2$\pm$0.6 & 7.7$\pm$3.7   & 4.3$\pm$0.1 & 10.6$\pm$2.0 \\
IMPALA-Macro  & 0.4$\pm$0.1 & 3.7$\pm$0.3 & 17.8$\pm$0.4  & 4.9$\pm$0.3 & 17.3$\pm$1.0 \\
MASP          & 0.7$\pm$0.2 & 1.1$\pm$0.3 & 6.9$\pm$6.9   & 2.7$\pm$0.1 & 3.0$\pm$1.0 \\
\bottomrule
\end{tabular}}
\end{table}

\paragraph{Results at $K{=}4$.}
As we can see from Table~\ref{tab:macro_k4}, DRIFT achieves the highest macro-action mean scores on Freeway (38.1 versus 32.8 for RL-D$^2$-O2O), Asterix (2.1), and Seaquest (7.0). Q-chunking leads the macro-action methods on Breakout (20.4), while RL-D$^2$-O2O leads on Space Invaders (40.2). Temporal abstraction is game-dependent: primitive R2D2 reaches 25.2 on Breakout, exceeding every macro-action method. Offline initialization improves DRIFT's final mean scores across all five games, including gains of 18.2 on Freeway and 9.3 on Space Invaders relative to online-only DRIFT.

\begin{figure}[h]
\centering
\includegraphics[width=\textwidth]{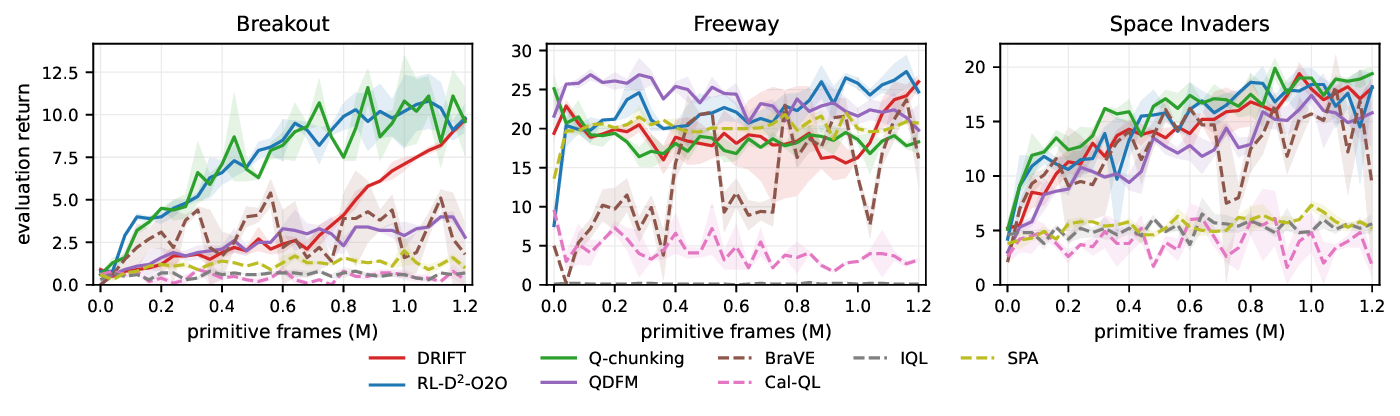}
\caption{Learning curves at $K{=}8$ ($|\mathcal{A}|{\approx}1.68$ million) for the methods with an offline phase, on the three games that separate methods. The first point is the score after pre-training and shading spans the five seeds. Table~\ref{tab:macro_k8} gives final scores for all games.}
\label{fig:macro_k8}
\end{figure}

\paragraph{Scaling to $K{=}8$.}
Here, the action space grows to approximately 1.68 million actions. Figure~\ref{fig:macro_k8} shows the offline-to-online learning curves on Breakout, Freeway, and Space Invaders. DRIFT remains competitive with the discrete generative baselines and achieves a higher final return on Freeway than RL-D$^2$-O2O (26.0 versus 24.7). Complete results across all five games are provided in Appendix~\ref{k8table}.


\paragraph{Ablations.}
Removing path-space KL reduces returns from 38.1 to 25.9 on Freeway and from 36.0 to 20.8 on Space Invaders.
Compared with terminal policy KL, path space KL improves Freeway returns from 34.9 to 38.1 and reduces path divergence from 0.39 to 0.21 (Table~\ref{tab:path_kl_main}). Although terminal policy KL achieves slightly lower terminal divergence, path space KL achieves higher returns and lower trajectory divergence.
These results support path space regularization as an important component of offline to online adaptation, although excessively strong regularization can restrict policy improvement.
A controlled comparison also finds similar returns for CTMC and categorical actors under terminal policy KL (34.9 versus 35.2), suggesting that the improvement is not simply due to greater representational capacity (Table~\ref{tab:actor_comparison}).

\begin{table}[h]
\centering
\caption{Comparison of KL regularization methods on Freeway at $K=4$ using five seeds. KL values are measured against the periodically refreshed reference policy. Coverage is the full action target probability mass captured by the candidate set.}
\label{tab:path_kl_main}
\begin{tabular}{lcccc}
\toprule
Method & Return $\uparrow$ & Terminal KL $\downarrow$ &
Path KL $\downarrow$ & Coverage $\uparrow$ \\
\midrule
No KL & $25.9\pm0.9$ & 0.44 & 0.71 & 0.08 \\
Terminal policy KL & $34.9\pm1.3$ &
\textbf{0.15} & 0.39 & 0.19 \\
Path space KL (DRIFT) & $\mathbf{38.1\pm0.5}$ &
0.18 & \textbf{0.21} & \textbf{0.26} \\
\bottomrule
\end{tabular}
\end{table}

\par\noindent
Replacing reference-guided candidates with uniform sampling and removing reference based target weighting further reduces returns to 11.1 and 8.0, respectively (Table~\ref{tab:macro_abl}).
On Freeway, reference only candidates achieve 37.2 compared with 38.1 for mixed candidates at $K=4$, with similar performance at $K=8$ (Table~\ref{tab:macro_candidate_composition}).
A joint ablation also shows positive interaction effects of 10.6 points on Freeway and 13.7 points on Space Invaders (Table~\ref{tab:joint_ablation_main}), indicating that KL regularization is more beneficial when combined with reference guidance.
Increasing the candidate budget generally improves performance, highlighting the trade-off between computation and candidate coverage.
The sensitivity trends for regularization strength and candidate budget also hold at $K=8$ (Table~\ref{tab:macro_k8_sensitivity}).
Full ablation results are provided in Appendix~\ref{app:macro_ablations}.


\begin{table}[h]
\centering
\caption{Joint ablation of path space KL and reference guidance on Freeway and Space Invaders at $K=4$ using five seeds. Returns are mean $\pm$ standard deviation. Removing reference guidance changes both candidate selection and reference based target weighting.}
\label{tab:joint_ablation_main}
\small
\begin{tabular}{lcc}
\toprule
Configuration & Freeway $\uparrow$ & Space Invaders $\uparrow$ \\
\midrule
Full DRIFT & $\mathbf{38.1\pm0.5}$ & $\mathbf{36.0\pm1.7}$ \\
Without path space KL & $25.9\pm0.9$ & $20.8\pm3.9$ \\
Without reference guidance & $11.1\pm1.6$ & $8.0\pm0.6$ \\
Without both & $9.5\pm1.4$ & $6.5\pm1.1$ \\
\bottomrule
\end{tabular}
\end{table}


\subsection{Text games without the admissible-action oracle}
\label{sec:jericho-noracle}

Prior work on Jericho, and our experiments in Appendix~\ref{app:jericho_details}, give every agent the admissible actions that Jericho computes by simulating candidates in the emulator. We remove that oracle, so an action is a template from the game grammar with up to two object slots filled with words from the observation, room description and inventory. DRRN-Struct isolates the generator, sharing DRIFT's action space, encoders, budget and offline data but selecting among uniform samples by Q-value. We also reimplement CALM~\citep{yao2020keep} and the self-imitation part of CSM~\citep{shi2023self} with GPT-2, and train DRRN~\citep{he2016deep} on admissible lists for reference. All methods use the same offline data and 100K online steps at $B{=}128$ with five seeds.

Without the admissible-action oracle, performance is substantially lower than the published oracle-based references (Table~\ref{tab:jmain}). DRIFT achieves a mean normalized score of 7.3\%, with final scores of 18.5 on Zork1 and 9.8 on Pentari, exceeding the other oracle-free methods on both games. DRRN-Struct scores higher on Detective (46.5 versus 42.8), while CSM leads on Ludicorp (8.4 versus 7.5). All oracle-free methods remain at 1.0 on Deephome, compared with 18.4 for DRRN using oracle actions. 
Relative to DRRN-Struct, DRIFT issues accepted commands
more frequently on Zork1 (0.262 versus 0.014) and
achieves higher admissible-action recall on four of
the five games. 
Increasing the candidate budget from $B=32$ to $B=256$
improves Detective recall from 0.123 to 0.519 and its
final score from 34.2 to 51.2.
Further results are provided in
Appendix~\ref{app:jericho-noracle}. Under matched runtime profiling, DRIFT also requires
approximately half the median training time of
DRRN-Struct (3.2 versus 6.4 hours);
see Appendix~\ref{app:jericho-noracle}.

\begin{table}[t]
\centering
\caption{Offline to online on five Jericho games without the admissible-action oracle, at $B{=}128$. Each cell reads as the score after offline pre-training, then the score after a further 100K online steps, the latter as mean $\pm$ std over 5 seeds. Bold marks the best final mean among methods without the oracle, and the best average gain. $^{\dagger}$From \citet{tuyls2022multi}, Table~1, which uses the oracle and 800K steps and reports no offline phase.}
\label{tab:jmain}
\small
\setlength{\tabcolsep}{3pt}
\resizebox{\textwidth}{!}{%
\begin{tabular}{lcccccc}
\toprule
Method & Zork1 & Deephome & Detective & Ludicorp & Pentari & Avg.\ \%max \\
\midrule
\multicolumn{7}{l}{\emph{No state-level admissible-action oracle}}\\
DRIFT (ours) & 1.2 $\rightarrow$ \textbf{18.5$\pm$1.5} & 1.0 $\rightarrow$ 1.0$\pm$0.0 & 20.8 $\rightarrow$ 42.8$\pm$5.2 & 6.0 $\rightarrow$ 7.5$\pm$0.1 & 1.6 $\rightarrow$ \textbf{9.8$\pm$1.2} & 2.5 $\rightarrow$ \textbf{7.3\% (+4.8)} \\
DRIFT-no path-KL & 1.4 $\rightarrow$ 14.0$\pm$2.0 & 1.0 $\rightarrow$ 1.0$\pm$0.0 & 23.5 $\rightarrow$ 40.0$\pm$4.0 & 3.2 $\rightarrow$ 5.6$\pm$0.1 & 1.9 $\rightarrow$ 3.8$\pm$0.2 & 2.4 $\rightarrow$ 4.9\% (+2.5) \\
DRIFT-no offline data & 11.4$\pm$1.3 & 1.0$\pm$0.0 & 20.0$\pm$4.0 & 3.6$\pm$0.1 & 2.8$\pm$0.2 & 3.1\\

DRRN-Struct & 0.2 $\rightarrow$ 12.0$\pm$1.5 & 1.0 $\rightarrow$ 1.0$\pm$0.0 & 26.5 $\rightarrow$ \textbf{46.5$\pm$1.0} & 2.1 $\rightarrow$ 2.3$\pm$0.1 & 0.5 $\rightarrow$ 0.0$\pm$0.0 & 2.0 $\rightarrow$ 3.6\% (+1.6) \\
CALM & 0.0 $\rightarrow$ 0.0$\pm$0.0 & 1.0 $\rightarrow$ 1.0$\pm$0.0 & 19.5 $\rightarrow$ 27.0$\pm$13.0 & 6.0 $\rightarrow$ 6.4$\pm$0.1 & 0.0 $\rightarrow$ 0.0$\pm$0.0 & 2.0 $\rightarrow$ 2.4\% (+0.4) \\
CSM & 0.0 $\rightarrow$ 0.0$\pm$0.0 & 1.0 $\rightarrow$ 1.0$\pm$0.0 & 19.5 $\rightarrow$ 41.2$\pm$1.2 & 6.0 $\rightarrow$ \textbf{8.4$\pm$0.9} & 0.0 $\rightarrow$ 0.0$\pm$0.0 & 2.0 $\rightarrow$ 3.5\% (+1.5) \\
\midrule
\multicolumn{7}{l}{\emph{With the admissible-action oracle, reference only}}\\
DRRN (oracle actions) & 10.0 $\rightarrow$ 12.3$\pm$2.3 & 20.2 $\rightarrow$ 18.4$\pm$1.1 & 40.2 $\rightarrow$ 89.0$\pm$16.0 & 15.6 $\rightarrow$ 16.0$\pm$1.1 & 15.4 $\rightarrow$ 30.0$\pm$0.0 & 10.6 $\rightarrow$ 17.6\% (+7.0) \\
DRRN$^{\dagger}$ & 40.3 & 58.8 & 289.9 & 17.1 & 45.6 & 37.6\% \\
XTX$^{\dagger}$ & 103.4 & 77.7 & 312.2 & 78.8 & 49.6 & 53.1\% \\
\midrule
Maximum game score & 350 & 300 & 360 & 150 & 70 & 100\% \\
\bottomrule
\end{tabular}
}
\end{table}

\paragraph{Additional results and takeaways.}
DRIFT also performs well on
standard MinAtar and discretized D4RL tasks
(Appendix~\ref{app:small_action}). Together with our macro-action
MinAtar and oracle-free Jericho results, these experiments demonstrate
the applicability of a common offline-to-online framework across
small, combinatorial, and language-based discrete action spaces. 
With only 22 actions in discretized D4RL, evaluating every action is inexpensive, so candidate selection offers no computational advantage in this setting. This mechanism is most useful and is designed for large discrete or combinatorial action spaces where evaluating all actions is costly as our macro-action MinAtar and oracle-free Jericho results show.
Our ablations further support the benefits of combining path-space
regularization with reference-guided candidate selection.

\section{Conclusion, Limitations and Future Work}
\label{sec:conclusion}

We introduced DRIFT for offline-to-online fine-tuning of discrete CTMC policies using advantage-weighted flow matching and path-space regularization. Our reference-guided candidate-set mechanism reduces computation in large action spaces, with theoretical bounds on its approximation error. On macro-action MinAtar, offline-pretrained DRIFT improves over its online-only counterpart across all five games and achieves the highest macro-action mean return on three at $K=4$. On Jericho without the admissible-action oracle, DRIFT achieves a mean normalized score of 7.3\%, outperforming the evaluated oracle-free baselines. DRIFT remains sensitive to the regularization weight and candidate-set coverage, and its performance varies across environments. Future work includes adaptive regularization, more efficient candidate selection, and evaluation on more combinatorial control tasks.

\small
\bibliographystyle{plainnat}
\bibliography{references}

\clearpage
\appendix

\appendix


\section*{Appendix}

\section{Additional Training Details}
\label{app:training}

\subsection{Offline Pre-Training}
\label{app:pretraining}

Our Stage~1 pre-training produces a reference generator $u_{\mathrm{ref}}$ from the offline dataset $\mathcal{D}_{\mathrm{off}}$.
This procedure is an \emph{advantage-weighted DFM} variant.

\paragraph{Critic pre-training.}
We pre-train two critic networks $Q_{\phi_1}, Q_{\phi_2}$ and a value network $V_\psi$ using standard TD learning on $\mathcal{D}_{\mathrm{off}}$:
\[
y_b = r_b + \gamma\,V_{\psi^-}(s_b')\,(1-d_b),
\qquad
\mathcal{L}_{Q_k}(\phi_k) = \frac{1}{B}\sum_b \bigl(Q_{\phi_k}(s_b,a_b) - \mathrm{sg}[y_b]\bigr)^2.
\]
The value network is trained via an advantage-weighted target.
This differs from the Boltzmann soft-value backup with in-support sampling.

\paragraph{Generator pre-training.}
We compute a target policy $\tilde{\pi}_{\mathrm{pre}}(a \mid s) \propto \exp(\bar{A}_{\mathrm{off}}(s,a)/\beta)$ over the \emph{full} action space using the pre-trained critics, where $\bar{A}_{\mathrm{off}}$ is the clipped, normalized pessimistic advantage. 
The generator is then trained via standard DFM to flow from $p_0 = \mathrm{Uniform}(\mathcal{A})$ to $\tilde{\pi}_{\mathrm{pre}}$ using the independent coupling (Eq.~\ref{eq:coupling}).

The pre-trained generator captures high-value behaviors from the offline data and provides a strong initialization for online fine-tuning.

\subsection{Critic and Value Network Updates}
\label{app:critic_value}

We maintain two independent critic networks $Q_{\phi_1},Q_{\phi_2}$ and a value network $V_\psi$, together with frozen target copies $Q_{\phi_k^-}$ and $V_{\psi^-}$.

\paragraph{Critic update.}
Each critic is trained by minimizing the squared Bellman error against a one-step TD target computed from the frozen value network:
\[
\mathcal{L}_{Q_k}(\phi_k)
=
\frac{1}{B}\sum_{b=1}^{B}
\left(
Q_{\phi_k}(s_b,a_b)
-
\mathrm{sg}\left[
r_b+\gamma V_{\psi^-}(s_b')(1-d_b)
\right]
\right)^2,
\]
where $\mathrm{sg}[\cdot]$ denotes stop-gradient. Using the frozen value network avoids a full CTMC rollout at every critic update.

\paragraph{Value update.}
The value network is trained using a lagged target policy constructed from frozen networks. We first compute
\[
A^-(s,a)
=
\min_k Q_{\phi_k^-}(s,a)-V_{\psi^-}(s).
\]
These advantages are normalized over the candidate set and clipped to $[-c,c]$, giving $\bar A^-(s,a)$. The lagged target policy is
\begin{equation}
\label{eq:lagged_target}
\tilde{\pi}^-(a\mid s)
\propto
\pi_{\mathrm{ref}}(a\mid s)
\exp\left(\bar A^-(s,a)/\beta\right),
\qquad
a\in\mathcal A_{\mathrm{cand}}(s).
\end{equation}
Both the lagged target policy and the value target are computed over the candidate set, avoiding full-action enumeration. The value loss is
\[
\mathcal L_V(\psi)
=
\frac{1}{B}\sum_{b=1}^{B}
\left(
V_\psi(s_b)
-
\mathrm{sg}\left[
\sum_{a\in\mathcal A_{\mathrm{cand}}(s_b)}
\tilde{\pi}^-(a\mid s_b)
\min_k Q_{\phi_k^-}(s_b,a)
\right]
\right)^2.
\]

All target networks are soft-updated:
\[
\phi_k^- \leftarrow
(1-\tau)\phi_k^-+\tau\phi_k,
\qquad
\psi^- \leftarrow
(1-\tau)\psi^-+\tau\psi.
\]

\section{CTMC Simulation Procedure}
\label{app:ctmc_sim}

\paragraph{Euler discretization.}
Euler discretization approximates a continuous-time Markov chain by dividing the time interval $[0,1]$ into $M$ small steps of size $\Delta t=1/M$. At each step $t_m=m\Delta t$, the process either stays at its current action $i$ or jumps to another action $j$. The probability of jumping from $i$ to $j$ is approximated by $\Delta t\,u_{t_m}(i\to j\mid s)$, while the probability of staying at $i$ is $1-\Delta t\sum_{j\neq i}u_{t_m}(i\to j\mid s)$. Repeating this procedure for $M$ steps produces the final action $X_M$. We require $\Delta t\sum_{j\neq i}u_{t_m}(i\to j\mid s)\leq 1$ to ensure valid transition probabilities.

Algorithm~\ref{alg:ctmc_sim} describes the Euler discretization used to sample actions from the learned CTMC generator.
This procedure is used both during environment interaction (Line~\ref{line:collect} of Algorithm~\ref{alg:sketch}) and when constructing candidate sets via reference rollouts (Line~\ref{line:cand}).

\begin{algorithm}[h]
\caption{CTMC Euler Simulation}
\label{alg:ctmc_sim}
\begin{algorithmic}[1]
\REQUIRE State $s$, generator $u$, sub-steps $M$, step size $\Delta t = 1/M$
\STATE Sample $X_0 \sim \mathrm{Uniform}(\mathcal{A})$
\FOR{$m = 0$ \TO $M-1$}
    \STATE Set $t_m \leftarrow m\,\Delta t$ and current action $i \leftarrow X_m$
    \STATE Compute total exit rate $\lambda \leftarrow \sum_{j \neq i} u(i\!\to\!j, t_m \mid s)$
    \IF{$\lambda = 0$}
        \STATE $X_{m+1} \leftarrow i$ \COMMENT{no transition possible}
    \ELSE
        \STATE With probability $1 - \lambda\,\Delta t$: $X_{m+1} \leftarrow i$ \COMMENT{stay}
        \STATE With probability $\lambda\,\Delta t$: sample $X_{m+1} \sim u(\cdot, t_m \mid s) / \lambda$ \COMMENT{jump}
    \ENDIF
\ENDFOR
\RETURN Terminal action $X_M$
\end{algorithmic}
\end{algorithm}

The scheme corresponds to the standard Euler discretization of a CTMC~\citep[Eq.~(6.6)]{lipman2024flow}. For sufficiently regular and bounded transition rates, it has a local error of $O(\Delta t^2)$ and a global error of $O(\Delta t)$. Validity requires $\lambda\Delta t \le 1$, which we ensure by choosing $M$ sufficiently large.

\section{Proofs and Derivations}
\label{app:proofs}

\subsection{Independent Coupling: Construction and Mass Conservation}
\label{app:coupling}

\paragraph{Construction.}
Given the linear bridge 
$p_t(a \mid s) = (1-t)\, p_0(a) + t\, \tilde\pi(a \mid s)$ 
with $p_0 = \mathrm{Uniform}(\mathcal{A})$, its time derivative is 
$\dot p_t(a) = \tilde\pi(a \mid s) - p_0(a)$. 
We decompose $\dot p_t$ into its positive and negative parts:
\[
(\dot p_t(a))_+ = \max(\dot p_t(a), 0), \quad
(\dot p_t(a))_- = \max(-\dot p_t(a), 0), \quad
Z_t = \sum_a (\dot p_t(a))_+.
\]
The independent coupling assigns rates proportional to the product of outflow from $i$ and inflow to $j$:
\[
u_t^*(i\!\to\!j \mid s) = \frac{(\dot{p}_t(i))^-\,(\dot{p}_t(j))^+}{p_t(i \mid s)\,Z_t},
\qquad i \neq j,
\]
with $u_t^*(i\!\to\!j \mid s) = 0$ whenever $p_t(i \mid s) = 0$ or $Z_t = 0$.
This construction is valid because it produces nonnegative off-diagonal rates and the diagonal is set by conservation.

\paragraph{Source actions outside the candidate set.}
The source distribution $p_0=\mathrm{Uniform}(\mathcal A)$ covers the full action space, whereas the target $\tilde\pi$ is supported only on $\mathcal A_{\mathrm{cand}}(s)$. Consequently, an intermediate action $i\sim p_t(\cdot\mid s)$ may lie outside the candidate set, especially for small $t$. For such an action,
\[
\tilde\pi(i\mid s)=0,\qquad
\dot p_t(i\mid s)=-\frac{1}{|\mathcal A|},
\qquad
(\dot p_t(i\mid s))^-=\frac{1}{|\mathcal A|}.
\]
Thus, $i$ acts as a source of probability mass. Under the independent-coupling generator in Eq.~\ref{eq:coupling}, this mass is transferred to actions $j$ with positive demand, satisfying
\[
(\dot p_t(j\mid s))^+>0
\quad\Longleftrightarrow\quad
\tilde\pi(j\mid s)>\frac{1}{|\mathcal A|}.
\]
Such actions must belong to $\mathcal A_{\mathrm{cand}}(s)$. Hence, actions outside the candidate set are not ignored. They supply probability mass to candidate actions that need to gain probability.

\paragraph{Interpretation in the fine-tuning context.}
The independent coupling spreads each source's outflow across all sinks proportionally to demand.
In the fine-tuning setting, this means that actions gaining probability under $\tilde{\pi}$ receive mass from \emph{all} actions losing probability, rather than from a single greedy source.
This produces a dense, stable supervision signal that reduces susceptibility to mode collapse during fine-tuning.

\begin{proposition}[Mass Conservation]
\label{prop:mass}
Assume $p_t(i) > 0$ for actions with outgoing mass and 
$Z_t = \sum_a (\dot p_t(a))_+ > 0$. Then the independent-coupling generator
\[
u^*_t(i \to j) = \frac{(\dot p_t(i))_- (\dot p_t(j))_+}{p_t(i)\, Z_t}, 
\qquad i \neq j,
\]
satisfies the Kolmogorov forward equation.
\end{proposition}

\paragraph{Proof of Proposition~\ref{prop:mass} (Mass Conservation).}
Substituting the generator into the Kolmogorov forward equation, the net flow into state $j$ is:
\begin{align*}
\mathrm{Net}(j)
&= \sum_{i \neq j} p_t(i)\,u_t^*(i\!\to\!j) - p_t(j)\sum_{k \neq j} u_t^*(j\!\to\!k) \\
&= \sum_{i \neq j} \frac{(\dot{p}_t(i))^-\,(\dot{p}_t(j))^+}{Z_t}
   - \frac{(\dot{p}_t(j))^-}{Z_t}\sum_{k \neq j} (\dot{p}_t(k))^+.
\end{align*}
Since $\sum_a \dot{p}_t(a) = 0$, total positive and negative parts are equal: $\sum_a (\dot{p}_t(a))^+ = \sum_a (\dot{p}_t(a))^- = Z_t$.
Therefore $\sum_{i \neq j} (\dot{p}_t(i))^- = Z_t - (\dot{p}_t(j))^-$ and $\sum_{k \neq j} (\dot{p}_t(k))^+ = Z_t - (\dot{p}_t(j))^+$.
Substituting and simplifying:
\[
\mathrm{Net}(j)
= \frac{(\dot{p}_t(j))^+\bigl(Z_t - (\dot{p}_t(j))^-\bigr) - (\dot{p}_t(j))^-\bigl(Z_t - (\dot{p}_t(j))^+\bigr)}{Z_t}
= (\dot{p}_t(j))^+ - (\dot{p}_t(j))^-
= \dot{p}_t(j). \qed
\]

\subsection{Path-Space KL Derivation}
\label{app:path_kl}

The KL divergence between CTMC path measures is a standard result from the theory of Markov jump processes~\citep{norris1998markov,kipnis2013scaling}.
We provide a self-contained derivation for completeness.

Let $\mathcal{P}_{u_\theta}$ and $\mathcal{P}_{u_{\mathrm{ref}}}$ denote the path measures of CTMCs with generators $u_\theta$ and $u_{\mathrm{ref}}$, respectively, both starting from the same initial distribution.
A CTMC path over $[0,1]$ consists of an initial state, a sequence of jump times $0 < t_1 < t_2 < \cdots < t_{N_{\mathrm{jump}}} \le 1$, and the corresponding jump destinations.
Between jumps, the chain holds at its current state.

The log-likelihood of a path under generator $u$ decomposes as:
\[
\log \mathcal{P}_u(\text{path})
= \sum_{k=1}^{N_{\mathrm{jump}}} \log u(X_{t_k^-}\!\to\!X_{t_k}, t_k \mid s)
- \int_0^1 \lambda_u(X_t, t \mid s)\,dt,
\]
where the first term accounts for the jumps that occurred and the second term accounts for the holding (no-jump) periods.
The KL divergence follows directly:
\begin{align*}
\mathrm{KL}(\mathcal{P}_{u_\theta}\!\parallel\!\mathcal{P}_{u_{\mathrm{ref}}})
&= \mathbb{E}_{\mathcal{P}_{u_\theta}}\!\left[
\log\frac{\mathcal{P}_{u_\theta}(\text{path})}{\mathcal{P}_{u_{\mathrm{ref}}}(\text{path})}
\right] \\
&= \mathbb{E}_{\mathcal{P}_{u_\theta}}\!\left[
\sum_{k:\mathrm{jumps}} \log\frac{u_\theta(X_{t_k^-}\!\to\!X_{t_k}, t_k)}{u_{\mathrm{ref}}(X_{t_k^-}\!\to\!X_{t_k}, t_k)}
+ \int_0^1 \bigl(\lambda_{\mathrm{ref}}(X_t,t) - \lambda_\theta(X_t,t)\bigr)dt
\right].
\end{align*}

\paragraph{Monte Carlo estimation.}
In practice, we estimate this using an equivalent compensator representation. For positive reference rates, define $d(u,v)=u\log(u/v)-u+v$. Then

\begin{equation}
\label{eq:kl_compensator}
\begin{aligned}
&\mathrm{KL}(\mathcal P_{u_\theta}\parallel
\mathcal P_{u_{\mathrm{ref}}})\\
&=\mathbb E_{\mathcal P_{u_\theta}}\left[
\int_0^1\sum_{j\ne X_t}
d\left(
u_\theta(X_t\to j,t\mid s),
u_{\mathrm{ref}}(X_t\to j,t\mid s)
\right)dt\right].
\end{aligned}
\end{equation}

Using a simulated CTMC trajectory with $M$ steps, $t_m=m/M$ and $\Delta t=1/M$, the corresponding Monte Carlo estimator is

\begin{equation}
\label{eq:kl_compensator_mc}
\widehat{\mathcal L}_{\mathrm{KL}}(s,\theta)
=\sum_{m=0}^{M-1}\Delta t
\sum_{j\ne X_m}
d\left(
u_\theta(X_m\to j,t_m\mid s),
u_{\mathrm{ref}}(X_m\to j,t_m\mid s)
\right).
\end{equation}

Sampled CTMC states are detached during differentiation. Since $\partial d(u,v)/\partial u=\log(u/v)$, gradients still flow through the evaluated transition rates. The omitted state occupancy derivatives make the resulting gradient biased. The compensator identity is exact in continuous time, while the finite step approximation can introduce additional error.

The compensator provides a useful regularization signal because its local gradient $\log(u_\theta/u_{\mathrm{ref}})$ pushes transition rates toward the reference even with detached trajectories. Its nonnegative Bregman form is closely related to the KL value estimator proposed by~\citet{schulman2020approximating} and used in GRPO~\citep{shao2024deepseekmath}, which is nonnegative and designed to have low variance. This connection does not imply gradient unbiasedness. Near the reference, the omitted occupancy-gradient contribution is second-order in the rate deviation under smoothness assumptions, which may help explain the observed gradient alignment. Related work uses generalized KL rate matching for CTMC fine-tuning through Discrete Adjoint Matching~\citep{so2026discrete}, and the same rate-based path-space KL identity for trajectory aware optimization of discrete diffusion models~\citep{tang2025tr2}. These connections provide theoretical intuition and related context, while our experiments directly validate the gradient behavior of our estimator.

For comparison, let $\ell_\theta(X)=\log(d\mathcal P_{u_\theta}/d\mathcal P_{u_{\mathrm{ref}}})(X)$. Under exact CTMC sampling, an unbiased score function gradient of the path space KL is $\mathbb E_{\mathcal P_{u_\theta}}[(1+\ell_\theta(X))\nabla_\theta\ell_\theta(X)]$.

Here we report evaluations on 20 four-state CTMC configurations against exact path-space KL gradients:

\begin{table}[h]
\centering
\caption{Comparison of KL gradient estimators on 20 four state CTMC configurations. Values are medians across configurations. The exact gradient serves as an oracle reference.}
\label{tab:kl_gradient_estimators}
\small
\resizebox{\linewidth}{!}{%
\begin{tabular}{lccccc}
\toprule
Estimator & Cosine $\uparrow$ & Rel. error $\downarrow$ &
Norm. $\Delta\mathrm{KL}$ $\downarrow$ &
Variance $\downarrow$ & $P[\mathrm{descent}]$ $\uparrow$ \\
\midrule
Exact gradient & 1.0000 & 0\% & $-0.986$ & 0 & 1.000 \\
Sampled estimator & N/A & 100\% & $+0.055$ & 31.791 & 0.501 \\
Compensator (ours) & 0.9971 & 8.3\% &
$\mathbf{-1.000}$ & $\mathbf{0.130}$ & $\mathbf{1.000}$ \\
Score function & 0.9985 & 6.9\% & $-0.819$ & 35.553 & 0.675 \\
Score function with baseline & 0.9985 & 6.9\% &
$-0.820$ & 34.550 & 0.672 \\
\bottomrule
\end{tabular}}
\end{table}

Cosine similarity and relative error compare mean gradients with exact gradients. Variance is normalized by the squared exact gradient norm, and $P[\mathrm{descent}]$ measures the fraction of sampled gradients aligned with the exact gradient.

The compensator reduced true path space KL in all 20 configurations, versus 4 of 20 for the sampled estimator. Its median gradient cosine similarity was 0.9971 with 8.3\% relative error. The relative error was 7.5\%, 6.0\%, and 15.7\% at rate deviations 0.25, 0.5, and 3.0, respectively, while maintaining successful KL reduction across the tested deviation levels.

The unbiased score function estimator achieved slightly better gradient alignment and lower relative error, but its gradient variance was 273.6 times that of the compensator, or 265.9 times with a leave one out baseline. These results support the compensator as a practical regularizer with a favorable balance between gradient bias and variance.

\subsection{Proof of Proposition~\ref{prop:coverage} (Coverage Error)}
\label{app:proof-coverage}

\begin{proof}
Let $\tilde\pi$ be the ideal full-action target policy and let
\[
\epsilon(s) = \sum_{a \notin \mathcal{A}_{\mathrm{cand}}(s)} \tilde\pi(a \mid s).
\]
Assume $\epsilon(s) < 1$. The candidate-restricted policy is
\[
\tilde\pi_{\mathrm{cand}}(a \mid s)
= \frac{\tilde\pi(a \mid s)\, \mathbf{1}\{a \in \mathcal{A}_{\mathrm{cand}}(s)\}}
       {1 - \epsilon(s)}.
\]
Therefore,
\[
\begin{aligned}
\|\tilde\pi_{\mathrm{cand}} - \tilde\pi\|_1
&= \sum_{a \in \mathcal{A}_{\mathrm{cand}}}
   \left| \frac{\tilde\pi(a)}{1 - \epsilon} - \tilde\pi(a) \right|
   + \sum_{a \notin \mathcal{A}_{\mathrm{cand}}} \tilde\pi(a) \\
&= \sum_{a \in \mathcal{A}_{\mathrm{cand}}}
   \tilde\pi(a) \cdot \frac{\epsilon}{1 - \epsilon} + \epsilon
= (1 - \epsilon) \cdot \frac{\epsilon}{1 - \epsilon} + \epsilon
= 2\epsilon.
\end{aligned}
\]
Under the idealized sampling model, action $a$ is excluded exactly when it is missed by all $N_{\mathrm{roll}}$ independent reference samples and all $N_{\mathrm{rand}}$ independent uniform samples. Hence
\[
\mathbb{P}(a \notin \mathcal{A}_{\mathrm{cand}}(s))
= (1 - \pi_{\mathrm{ref}}(a \mid s))^{N_{\mathrm{roll}}}
   (1 - 1/|\mathcal{A}|)^{N_{\mathrm{rand}}}.
\]
Thus,
\[
\mathbb{E}[\epsilon(s)]
= \sum_{a \in \mathcal{A}}
  \tilde\pi(a \mid s)
  (1 - \pi_{\mathrm{ref}}(a \mid s))^{N_{\mathrm{roll}}}
   (1 - 1/|\mathcal{A}|)^{N_{\mathrm{rand}}},
\]
and the result follows from
$\|\tilde\pi_{\mathrm{cand}} - \tilde\pi\|_1 = 2\epsilon(s)$.
\end{proof}

\subsection{Proof of~\cref{thm:generator-stability} (Generator Stability)}
\label{app:proof-stability}

\begin{proof}
The proof proceeds in three parts. First, we set up the notation and 
record elementary boundedness facts that will be used throughout. 
Second, we decompose the generator difference into three terms, 
each measuring a distinct source of perturbation. Third, we bound 
each term separately and combine the results.

\paragraph{Step 1: Setup and elementary bounds.}

Let
\[
d = \|\tilde\pi_{\mathrm{cand}} - \tilde\pi\|_1.
\]
By construction, both the full and candidate bridges use the same 
source $p_0 = \mathrm{Uniform}(\mathcal{A})$ (Eq.~\ref{eq:linear-bridge}). 
Therefore,
\[
\dot p_t^{\mathrm{cand}} - \dot p_t = \tilde\pi_{\mathrm{cand}} - \tilde\pi,
\qquad
\|\dot p_t^{\mathrm{cand}} - \dot p_t\|_1 = d.
\]

For any fixed source action $i$ and destination action $j \neq i$, let 
$p_t(i)$ and $p_t^{\mathrm{cand}}(i)$ denote the full and candidate 
bridge probabilities of action $i$ at flow time $t$. The maps 
$x \mapsto x_+ = \max\{x, 0\}$ and $x \mapsto x_- = \max\{-x, 0\}$ are 
$1$-Lipschitz \citep{rudin1976principles}, so
\[
\sum_a \big| (\dot p_t^{\mathrm{cand}}(a))_+ - (\dot p_t(a))_+ \big| \leq d,
\qquad
\sum_a \big| (\dot p_t^{\mathrm{cand}}(a))_- - (\dot p_t(a))_- \big| \leq d.
\]
The normalizers are
\[
Z_t = \sum_a (\dot p_t(a))_+, 
\qquad 
Z_t^{\mathrm{cand}} = \sum_a (\dot p_t^{\mathrm{cand}}(a))_+,
\]
which equal the corresponding total outgoing masses by 
$\sum_a \dot p_t(a) = 0$. Hence 
$Z_t = \tfrac{1}{2}\|\dot p_t\|_1$ and similarly for $Z_t^{\mathrm{cand}}$.

To simplify notation, we abbreviate the four quantities entering each 
rate (the source supply, sink demand, source probability, and total 
moving mass) using unsubscripted letters for the candidate bridge and 
a $0$-subscript for the full bridge:
\[
A := (\dot p_t^{\mathrm{cand}}(i))_-, 
\quad A_0 := (\dot p_t(i))_-, 
\quad B_j := (\dot p_t^{\mathrm{cand}}(j))_+, 
\quad B_j^0 := (\dot p_t(j))_+,
\]
and
\[
P := p_t^{\mathrm{cand}}(i), 
\quad P_0 := p_t(i), 
\quad Z := Z_t^{\mathrm{cand}}, 
\quad Z_0 := Z_t.
\]
With this notation,
\[
u_t^{*,\mathrm{cand}}(i \to j) = \frac{A B_j}{P Z}, 
\qquad 
u_t^*(i \to j) = \frac{A_0 B_j^0}{P_0 Z_0}.
\]
The lower bounds in the theorem statement give 
$P, P_0 \geq \underline{p} > 0$ and $Z, Z_0 \geq \underline{Z} > 0$. 
Furthermore, since $\|\dot p_t\|_1 \leq \|\tilde\pi\|_1 + \|p_0\|_1 = 2$, 
we have
\[
A,\, A_0,\, B_j,\, B_j^0 \in [0, 1], 
\qquad 
Z,\, Z_0 \in [0, 1], 
\qquad 
P,\, P_0 \in [0, 1].
\]
These elementary boundedness facts will be used repeatedly below.

\paragraph{Step 2: Three-term decomposition.}

Inserting and subtracting two intermediate terms,
\[
\frac{A B_j}{P Z} - \frac{A_0 B_j^0}{P_0 Z_0}
= \frac{(A - A_0) B_j}{P Z}
+ \frac{A_0 (B_j - B_j^0)}{P Z}
+ A_0 B_j^0 \left( \frac{1}{P Z} - \frac{1}{P_0 Z_0} \right).
\]
Summing over $j \neq i$ and applying the triangle inequality,
\[
\sum_{j \neq i} \big| u_t^{*,\mathrm{cand}}(i \to j) - u_t^*(i \to j) \big|
\leq \mathrm{(I)} + \mathrm{(II)} + \mathrm{(III)},
\]
where
\[
\mathrm{(I)} := \frac{|A - A_0|}{P Z} \sum_{j \neq i} B_j, 
\quad
\mathrm{(II)} := \frac{A_0}{P Z} \sum_{j \neq i} \big| B_j - B_j^0 \big|, 
\quad
\mathrm{(III)} := A_0 \sum_{j \neq i} B_j^0 \left| \frac{1}{P Z} - \frac{1}{P_0 Z_0} \right|.
\]
Each term measures a distinct perturbation: (I) the change in source 
supply, (II) the change in sink demand, and (III) the change in the 
denominator.

\paragraph{Step 3a: Bound on Term (I).}

Since $\sum_{j \neq i} B_j \leq Z$, the factor $Z$ cancels:
\[
\mathrm{(I)} \leq \frac{|A - A_0|}{P}.
\]
Using $P \geq \underline{p}$ and the $1$-Lipschitz property of $x \mapsto x_-$,
\[
|A - A_0| 
= \big| (\dot p_t^{\mathrm{cand}}(i))_- - (\dot p_t(i))_- \big| 
\leq \big| \dot p_t^{\mathrm{cand}}(i) - \dot p_t(i) \big| 
= \big| \tilde\pi_{\mathrm{cand}}(i) - \tilde\pi(i) \big| 
\leq d,
\]
where the final inequality uses that the pointwise difference is dominated 
by the $L_1$ norm. Therefore,
\[
\mathrm{(I)} \leq \frac{d}{\underline{p}}.
\]

\paragraph{Step 3b: Bound on Term (II).}

Since $A_0 \leq 1$, $P \geq \underline{p}$, and $Z \geq \underline{Z}$,
\[
\mathrm{(II)} \leq \frac{1}{\underline{p}\, \underline{Z}} \sum_{j \neq i} \big| B_j - B_j^0 \big|.
\]
The $1$-Lipschitz property of $x \mapsto x_+$ gives 
$|B_j - B_j^0| \leq |\dot p_t^{\mathrm{cand}}(j) - \dot p_t(j)| 
= |\tilde\pi_{\mathrm{cand}}(j) - \tilde\pi(j)|$ for each $j$, so 
summing yields $\sum_{j \neq i} |B_j - B_j^0| \leq d$. Therefore,
\[
\mathrm{(II)} \leq \frac{d}{\underline{p}\, \underline{Z}}.
\]

\paragraph{Step 3c: Bound on Term (III).}

We expand the denominator difference as
\[
\frac{1}{P Z} - \frac{1}{P_0 Z_0} 
= \frac{P_0 Z_0 - P Z}{P Z\, P_0 Z_0}
= \frac{(P_0 - P) Z_0 + P (Z_0 - Z)}{P Z\, P_0 Z_0}.
\]
Taking absolute values, the denominator is bounded below by 
$P Z\, P_0 Z_0 \geq \underline{p}^{\,2}\, \underline{Z}^{\,2}$. 
For the numerator, since $P \leq 1$ and $Z_0 \leq 1$,
\[
\big| (P_0 - P) Z_0 + P (Z_0 - Z) \big| 
\leq |P_0 - P| + |Z_0 - Z|.
\]
Using the shared-source structure of the bridges,
\[
P_0 - P = p_t(i) - p_t^{\mathrm{cand}}(i) 
= t\, (\tilde\pi(i) - \tilde\pi_{\mathrm{cand}}(i)),
\]
so $|P_0 - P| \leq |\tilde\pi(i) - \tilde\pi_{\mathrm{cand}}(i)| \leq d$ 
(pointwise dominated by $L_1$). Likewise, the $1$-Lipschitz property 
applied coordinate-wise and summed gives
\[
|Z_0 - Z| 
= \left| \sum_a (\dot p_t(a))_+ - \sum_a (\dot p_t^{\mathrm{cand}}(a))_+ \right|
\leq \sum_a \big| (\dot p_t(a))_+ - (\dot p_t^{\mathrm{cand}}(a))_+ \big|
\leq d.
\]
Hence the numerator is bounded by $2d$. The prefactor satisfies 
$A_0 \leq 1$ and $\sum_{j \neq i} B_j^0 \leq Z_0 \leq 1$, so 
$A_0 \sum_{j \neq i} B_j^0 \leq 1$. Therefore,
\[
\mathrm{(III)} 
\leq \frac{2d}{\underline{p}^{\,2}\, \underline{Z}^{\,2}}.
\]

\paragraph{Step 4: Combining the three bounds.}

Adding the three contributions,
\[
\sum_{j \neq i} \big| u_t^{*,\mathrm{cand}}(i \to j) - u_t^*(i \to j) \big|
\leq \frac{d}{\underline{p}} 
+ \frac{d}{\underline{p}\, \underline{Z}} 
+ \frac{2d}{\underline{p}^{\,2}\, \underline{Z}^{\,2}}.
\]
Since $\underline{p}, \underline{Z} \in (0, 1]$,
\[
\frac{1}{\underline{p}} \leq \frac{1}{\underline{p}^{\,2}\, \underline{Z}^{\,2}}, 
\qquad 
\frac{1}{\underline{p}\, \underline{Z}} \leq \frac{1}{\underline{p}^{\,2}\, \underline{Z}^{\,2}}.
\]
All three terms can therefore be upper bounded by the common denominator, 
yielding
\[
\sum_{j \neq i} \big| u_t^{*,\mathrm{cand}}(i \to j) - u_t^*(i \to j) \big|
\leq \frac{4}{\underline{p}^{\,2}\, \underline{Z}^{\,2}}\, 
\|\tilde\pi_{\mathrm{cand}} - \tilde\pi\|_1.
\]
This establishes the theorem with universal constant $C = 4$.
\end{proof}

\begin{remark}
The theorem isolates the error caused by restricting the terminal target 
policy to $\mathcal{A}_{\mathrm{cand}}$. When the candidate bridge also 
changes the source distribution from $p_0$ to $p_0^{\mathrm{cand}}$, an 
additional term $\|p_0^{\mathrm{cand}} - p_0\|_1$ enters the bound through 
$\|\dot p_t^{\mathrm{cand}} - \dot p_t\|_1$.
\end{remark}

\begin{corollary}[Uniform-in-Time Stability]
\label{cor:uniform_stability}
Assume that the full and candidate bridges use the same source 
distribution $p_0$. Fix $\delta \in (0, 1)$ and suppose that, for all 
$t \in [0, 1-\delta]$,
\[
p_t(i),\; p_t^{\mathrm{cand}}(i) \geq \underline{p} > 0, 
\qquad 
Z_t,\; Z_t^{\mathrm{cand}} \geq \underline{Z} > 0.
\]
Then, for every such $t$ and every $i \in \mathcal{A}_{\mathrm{cand}}$,
\[
\sum_{j \neq i} \big| u_t^{*,\mathrm{cand}}(i \to j) - u_t^*(i \to j) \big|
\leq \frac{C}{\underline{p}^{\,2}\, \underline{Z}^{\,2}}\, 
\|\tilde\pi_{\mathrm{cand}} - \tilde\pi\|_1,
\]
where $C > 0$ is a universal constant.

If $\underline{p}$ and $\underline{Z}$ are deterministic constants, 
then combining this bound with Proposition~\ref{prop:coverage} gives
\[
\mathbb{E}\!\left[
\sum_{j \neq i} \big| u_t^{*,\mathrm{cand}}(i \to j) - u_t^*(i \to j) \big|
\right]
\leq \frac{2C}{\underline{p}^{\,2}\, \underline{Z}^{\,2}}
\sum_{a \in \mathcal{A}} \tilde\pi(a \mid s)
(1 - \pi_{\mathrm{ref}}(a \mid s))^{N_{\mathrm{roll}}}
 (1 - 1/|\mathcal{A}|)^{N_{\mathrm{rand}}}.
\]
\end{corollary}

\begin{remark}
\label{rem:lower_bound}
Using a uniform prior $p_0$ and truncating the flow at $t \leq 1 - \delta$ 
ensures $p_t(i) \geq \delta / |\mathcal{A}|$, providing a practical lower 
bound for $\underline{p}$.
\end{remark}


\section{Additional Macro-Action MinAtar Experiments}
\label{app:macro_details}

\subsection{Construction and protocol}
A macro of length $K$ maps to a $K$-tuple of primitive actions via base-6 decomposition and executes them in sequence, accumulating reward. The agent observes only the state after the last one. If the episode ends part-way through, the remaining primitives are not executed. All methods use semi-Markov targets, so the reward is $\sum_{t<n}\gamma^{t}r_t$ for the $n$ frames executed and the bootstrap is discounted by $\gamma^{n}$. Episodes are capped at 5,000 primitive frames, which is treated as termination for bootstrapping. The cap is recorded so that truncations can be distinguished offline.

The offline dataset for each (game, $K$) contains 100K macro transitions. A behaviour double DQN is trained online on the \emph{primitive} action set with $\varepsilon$ annealed from 1.0 to 0.1 over the first half of collection, and its action stream is grouped into consecutive $K$-macros within each episode. A chunk that ends early is kept and padded with no-ops to length $K$. This is the correct semi-Markov transition for the padded macro and preserves terminal transitions, which dropping incomplete chunks would remove. The dataset is generated once per (game, $K$) from a fixed seed and shared by every method and seed.

\subsection{Results at $K{=}8$}
\label{k8table}

At $K=8$, the action space grows to approximately 1.68 million actions. DRIFT achieves the highest offline-to-online macro-action returns on Asterix, Freeway, and Seaquest, while RL-D$^2$-O2O and Q-chunking obtain slightly higher scores on Breakout and Space Invaders, respectively. Offline initialization substantially improves DRIFT over its online-only counterpart across all five games, including Freeway (26.0 versus 7.2) and Space Invaders (18.1 versus 2.8). However, primitive-action methods remain stronger on several games, highlighting the task-dependent benefits of temporal abstraction. Primitive-action DQN~\citep{mnih2015human}, IMPALA, PPO~\citep{schulman2017proximal} and R2D2~\citep{kapturowski2018recurrent} measure the value of temporal abstraction itself.

\begin{table}[h]
\centering
\caption{Macro-action MinAtar at $K{=}8$ ($|\mathcal{A}|{=}1{,}679{,}616$), same format as Table~\ref{tab:macro_k4}. MASP is omitted because its similarity matrix would have $|\mathcal{A}|^2\approx2.8\times10^{12}$ entries.}
\label{tab:macro_k8}
\resizebox{\textwidth}{!}{%
\begin{tabular}{lccccc}
\toprule
Method & Asterix & Breakout & Freeway & Seaquest & Space Invaders \\
\midrule
\multicolumn{6}{l}{\textit{Macro-action, offline-to-online}} \\
DRIFT (ours) & 0.6 $\to$ \textbf{2.0$\pm$0.1} & 0.9 $\to$ 9.7$\pm$0.4 & 19.4 $\to$ \textbf{26.0$\pm$0.8} & 0.5 $\to$ \textbf{1.9$\pm$0.1} & 5.2 $\to$ 18.1$\pm$0.9 \\
RL-D$^2$-O2O      & 0.5 $\to$ 1.1$\pm$0.0 & 0.8 $\to$ \textbf{9.8$\pm$0.9} & 7.6 $\to$ 24.7$\pm$0.5 & 0.2 $\to$ 0.8$\pm$0.1 & 4.2 $\to$ 18.2$\pm$1.2 \\
Q-chunking        & 0.6 $\to$ 0.7$\pm$0.0 & 0.6 $\to$ 9.6$\pm$2.0 & 25.1 $\to$ 18.3$\pm$1.5 & 0.7 $\to$ 0.6$\pm$0.1 & 5.1 $\to$ \textbf{19.4$\pm$0.6} \\
QDFM              & 0.8 $\to$ 1.0$\pm$0.2 & 0.0 $\to$ 2.8$\pm$0.1 & 21.6 $\to$ 19.8$\pm$3.2 & 0.6 $\to$ 1.3$\pm$0.2 & 3.0 $\to$ 15.8$\pm$3.1 \\
BraVE             & 0.4 $\to$ 0.5$\pm$0.1 & 0.0 $\to$ 1.8$\pm$1.8 & 5.0 $\to$ 16.2$\pm$3.6 & 0.0 $\to$ 0.4$\pm$0.1 & 2.1 $\to$ 9.2$\pm$7.8 \\
Cal-QL            & 0.4 $\to$ 0.7$\pm$0.1 & 0.6 $\to$ 0.3$\pm$0.3 & 9.6 $\to$ 3.2$\pm$0.2  & 0.3 $\to$ 0.0$\pm$0.0 & 2.6 $\to$ 1.8$\pm$1.0 \\
IQL               & 0.3 $\to$ 0.5$\pm$0.1 & 0.3 $\to$ 0.7$\pm$0.1 & 0.1 $\to$ 0.1$\pm$0.1  & 0.0 $\to$ 0.1$\pm$0.0 & 4.6 $\to$ 5.6$\pm$0.9 \\
SPA               & 0.5 $\to$ 0.6$\pm$0.1 & 0.6 $\to$ 1.0$\pm$0.2 & 13.6 $\to$ 20.7$\pm$0.2 & 0.2 $\to$ 0.4$\pm$0.1 & 3.8 $\to$ 5.2$\pm$0.5 \\
\midrule
\multicolumn{6}{l}{\textit{Macro-action, online only}} \\
DRIFT, no offline data & 0.9$\pm$0.1 & 3.6$\pm$0.0 & 7.2$\pm$0.3 & 0.5$\pm$0.0 & 2.8$\pm$0.4 \\
RL-D$^2$     & 0.8$\pm$0.3 & 4.5$\pm$0.1 & 10.8$\pm$10.8 & 0.8$\pm$0.0 & 23.9$\pm$1.7 \\
DQN-Macro    & 0.6$\pm$0.1 & 7.4$\pm$1.4 & 0.9$\pm$0.9   & 0.3$\pm$0.0 & 4.5$\pm$2.0 \\
IMPALA-Macro & 0.2$\pm$0.1 & 3.2$\pm$0.1 & 8.0$\pm$7.6   & 0.7$\pm$0.1 & 16.2$\pm$1.4 \\
\midrule
\multicolumn{6}{l}{\textit{Primitive actions (reference)}} \\
DQN    & 0.6$\pm$0.2 & 24.7$\pm$4.4 & 4.0$\pm$0.7   & 2.5$\pm$0.2 & 44.1$\pm$18.2 \\
IMPALA & 0.9$\pm$0.2 & 4.6$\pm$0.4  & 22.3$\pm$1.3  & 1.9$\pm$0.1 & 28.0$\pm$2.3 \\
PPO    & 2.2$\pm$0.2 & 5.3$\pm$0.1  & 19.8$\pm$19.8 & 2.6$\pm$1.2 & 40.6$\pm$3.1 \\
R2D2   & 0.6$\pm$0.2 & 25.2$\pm$0.4 & 12.4$\pm$0.9  & 1.4$\pm$0.1 & 23.4$\pm$2.9 \\
\bottomrule
\end{tabular}}
\end{table}

\subsection{Learning curves at $K{=}4$}

Figure~\ref{fig:macro_k4} shows that performance varies across games. Q-chunking leads on Breakout, while DRIFT achieves the highest final return on Freeway. On Space Invaders, RL-D$^2$-O2O overtakes DRIFT during online fine-tuning and achieves a higher final return. Overall, DRIFT improves throughout training, although no method consistently dominates across all three games.

\begin{figure}[h]
\centering
\includegraphics[width=\textwidth]{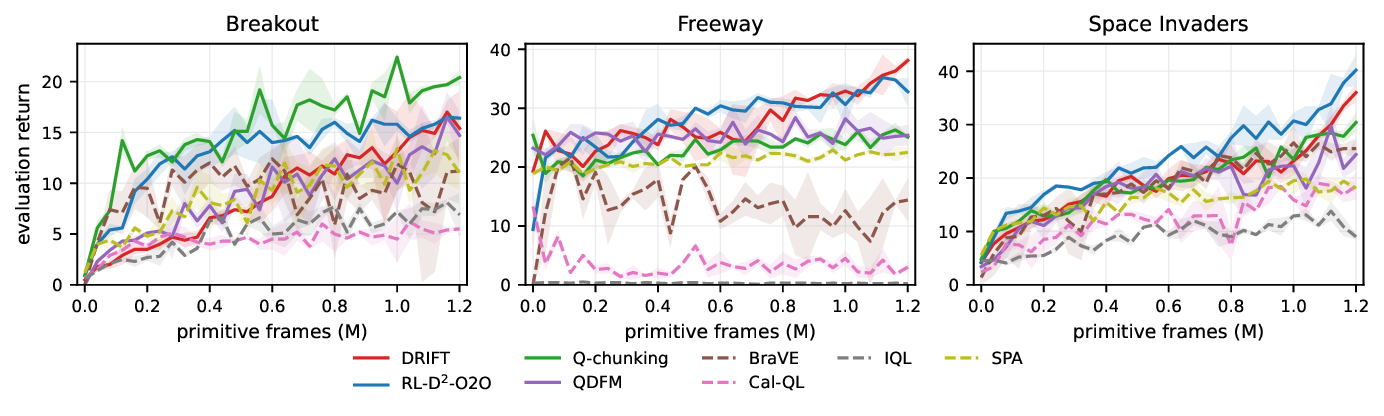}
\caption{Offline-to-online learning curves at $K{=}4$ on the three games that separate methods. The first point of each curve is the score after offline pre-training, and shading spans the five seeds.}
\label{fig:macro_k4}
\end{figure}

\subsection{Candidate coverage}
For every DRIFT run we log the mean candidate-set size, the fraction of fresh reference-generator rollouts whose terminal action lies inside the candidate set, and, where the action space is enumerable ($K{=}4$), the exact mass that the candidate set captures of the full-space target $\tilde\pi\propto\exp(\bar A/\beta)$. Table~\ref{tab:macro_cov} reports these against the budget on Freeway.

Normalized generator error measures the relative difference between the ideal CTMC transition rates induced by the candidate and full action targets. Advantage gap measures the loss in maximum estimated advantage due to candidate restriction, using the same full action normalization. Lower values are better.

\begin{table}[h]
\centering
\caption{Candidate budget on Freeway at $K{=}4$, $\alpha{=}0.01$, five seeds. Target mass is the exact probability of the candidate set under the full-space target, and $B/|\mathcal{A}|$ is what uniformly random candidates would capture.}
\label{tab:macro_cov}
\small
\resizebox{\linewidth}{!}{%
\begin{tabular}{lcccccc}
\toprule
$B$ & Final score & Target mass & $B/|\mathcal{A}|$ &
Mean $|\mathcal{A}_{\mathrm{cand}}|$ &
Gen. error & Adv. gap \\
\midrule
32  & 36.9$\pm$0.2 & 0.15 & 0.025 & 32 &
0.96 & 0.78 \\
64  & 38.1$\pm$0.5 & 0.26 & 0.049 & 64 &
0.73 & 0.29 \\
128 & 38.6$\pm$1.6 & 0.33 & 0.099 & 128 &
0.64 & 0.17 \\
256 & 39.6$\pm$1.0 & 0.38 & 0.198 & 256 &
0.55 & 0.11 \\
\midrule
Full & $41.0\pm0.2$ & 1.00 &
1.000 & 1296 & 0 & 0 \\
\bottomrule
\end{tabular}}
\end{table}

Although the candidate set captures only 26\% of the
measured full-action target mass at the default budget
of 64, DRIFT achieves a Freeway return of 38.1.
Increasing the budget to 256 raises the captured mass
to 38\% and the return to 39.6. These results
suggest that effective control does not necessarily
require closely approximating the entire full-action
target distribution.

At the default budget of $B=64$, the excluded target mass of 0.74 gives an exact target policy $L_1$ error of 1.48 by Proposition~\ref{prop:coverage}. Increasing $B$ from 32 to 256 reduces normalized generator error from 0.96 to 0.55 and the advantage gap from 0.78 to 0.11. These diagnostics show that the approximation improves with larger candidate budgets, even when substantial target mass remains excluded. So accurate approximation of the full target is sufficient for a small approximation error under the theorem's assumptions, but is not necessary for strong empirical performance.

Theorem~\ref{thm:generator-stability} provides a conditional approximation guarantee, but its bound can be loose in large action spaces. For the uniform source linear bridge and $t\leq 1-\delta$, we have $p_t(i)\geq\delta/|\mathcal A|$. Using this lower bound gives a coefficient $C|\mathcal A|^2/(\delta^2 Z^2)$ in Theorem~\ref{thm:generator-stability}. Thus, the theorem does not provide a dimension independent guarantee.

Our additional diagnostics on Space Invaders at $K=4$ show that increasing $B$ from 32 to 256 reduces normalized generator error from 0.91 to 0.62 and the advantage gap from 0.43 to 0.07, while returns improve from 33.5 to 43.6. We also evaluated $K=8$, where $|\mathcal A|=1{,}679{,}616$. Increasing $B$ from 32 to 256 improves returns from 24.1 to 28.3 on Freeway and from 16.6 to 22.4 on Space Invaders (Table~\ref{tab:macro_k8_sensitivity}). These results support practical scalability without establishing that the theoretical bound remains tight.

The required candidate budget depends on target policy concentration, not just action space size. Jericho also shows improved candidate recall with larger budgets in Table~\ref{tab:jbudget}, although recall differs from target mass coverage. An adaptive rule could increase the budget until estimated coverage gains become small. However, estimating these gains without evaluating the full action space remains challenging, so we use a fixed budget.

\subsection{Ablation Studies}
\label{app:macro_ablations}

We examine the effects of path-space regularization,
candidate-set construction, and candidate budget at $K=4$.

\begin{table}[ht]
\centering
\caption{DRIFT ablations at $K{=}4$, final score as mean $\pm$ std over 5 seeds. The default uses $\alpha{=}0.01$ and 64 candidates from reference-generator rollouts. Reference-free candidates keep the dataset action and fill the set uniformly.}
\label{tab:macro_abl}
\small
\begin{tabular}{llcc}
\toprule
Component & Setting & Freeway & Space Invaders \\
\midrule
\multirow{6}{*}{Path-space KL}
 & DRIFT, $\alpha{=}0$     & 25.9$\pm$0.9 & 20.8$\pm$3.9 \\
 & DRIFT, $\alpha{=}0.001$           & 28.9$\pm$1.9 & 23.4$\pm$0.7 \\
 & DRIFT, $\alpha{=}0.01$ (default)            & 38.1$\pm$0.5 & 36.0$\pm$1.7 \\
 & DRIFT, $\alpha{=}0.03$            & 29.5$\pm$0.9 & 27.1$\pm$0.6 \\
 & DRIFT, $\alpha{=}0.1$ & 26.8$\pm$0.6  & 26.0$\pm$1.1 \\
 & DRIFT, $\alpha{=}0.3$             & 26.8$\pm$0.3 & 21.2$\pm$0.1 \\
\midrule
\multirow{3}{*}{Candidate set}
 & Rollout candidates, $B{=}64$ (default) & 38.1$\pm$0.5 & 36.0$\pm$1.7 \\
 & Full action space (all 1,296)          & 41.0$\pm$0.2 & 24.6$\pm$1.5 \\
 & Reference-free candidates              & 11.1$\pm$1.6 & 8.0$\pm$0.6 \\
\midrule
\multirow{4}{*}{Budget $B$}
 & $B{=}32$  & 36.9$\pm$0.2 & 33.5$\pm$0.3 \\
 & $B{=}64$  & 38.1$\pm$0.5 & 36.0$\pm$1.7 \\
 & $B{=}128$ & 38.6$\pm$1.6 & 41.0$\pm$2.9 \\
 & $B{=}256$ & 39.6$\pm$1.0 & 43.6$\pm$1.8 \\
\bottomrule
\end{tabular}
\end{table}

\paragraph{Results.}
Table~\ref{tab:macro_abl} shows that $\alpha=0.01$
achieves the highest returns among the tested
regularization weights on both games. Removing
reference-guided candidates substantially reduces
performance, while increasing the candidate budget
consistently improves returns. Full-action enumeration
improves Freeway but performs worse on Space Invaders,
demonstrating that evaluating more actions does not
necessarily improve policy performance.
This does not contradict Theorem~\ref{thm:generator-stability}, which bounds target generator approximation error rather than final policy returns. In our implementation, changing the candidate set also changes advantage normalization and smoothed reference weights. Thus, full enumeration changes the learning target and does not simply remove approximation error from an otherwise fixed training objective.

\paragraph{Candidate composition.}
We compared mixed, reference only, and reference free
candidate construction on Freeway at $K=4$ and $K=8$.
All configurations use a maximum candidate budget of
$B=64$. Reference only removes uniform filling and may
contain fewer than 64 unique actions. The reference free
configuration also replaces reference frequency weighting
with a uniform target prior.

\begin{table}[h]
\centering
\caption{Candidate composition on Freeway with $B=64$.
Returns are mean $\pm$ standard deviation over five seeds.
Provenance diagnostics describe the mixed candidate set.}
\label{tab:macro_candidate_composition}
\small
\begin{tabular}{lcc}
\toprule
Configuration or diagnostic & $K=4$ & $K=8$ \\
\midrule
\multicolumn{3}{l}{\textit{Final return}} \\
DRIFT mixed & $\mathbf{38.1\pm0.5}$ & $26.0\pm0.8$ \\
Reference only & $37.2\pm0.7$ & $\mathbf{26.1\pm0.7}$ \\
Reference free & $11.1\pm1.6$ & $3.6\pm0.8$ \\
\midrule
\multicolumn{3}{l}{\textit{Candidate provenance}} \\
Slots filled uniformly & 90\% & 83\% \\
Uniform exclusive action has highest advantage
& 3.3\% & 0.05\% \\
Target mass on uniform exclusive actions
& 1.7\% & 0.04\% \\
\bottomrule
\end{tabular}
\end{table}

Although uniform sampling fills 90\% of candidate
slots at $K=4$, uniform exclusive actions receive only
1.7\% of the candidate restricted learning target mass.
At $K=8$, these values are 83\% and 0.04\%.
The reference only configuration achieves nearly the
same returns as mixed sampling, suggesting that uniform
filling provides limited additional benefit in these
experiments.

This also explains why a candidate set containing
mostly uniform samples can still benefit strongly from
reference guidance. At $B=64$, mixed candidates capture
26\% of the full action diagnostic target mass, compared
with 4.9\% expected from uniform sampling
(Table~\ref{tab:macro_cov}). These coverage values differ
from the candidate restricted target mass reported above.
The reference free ablation also changes target weighting,
so its lower returns cannot be attributed to candidate
selection alone.

\paragraph{Joint component ablation.}
To examine whether path space KL and reference guidance provide complementary benefits, we evaluated their combinations on Freeway and Space Invaders at $K=4$ using five seeds (Table~\ref{tab:joint_ablation_main}).


The combined method achieves the highest returns on both games. The interaction effects are $38.1-25.9-11.1+9.5=10.6$ points on Freeway and $36.0-20.8-8.0+6.5=13.7$ points on Space Invaders. These positive interactions show that path space KL provides greater improvement when combined with reference guidance. The reference free configuration changes both candidate selection and reference based target weighting, so this comparison evaluates their combined contribution.

We also evaluated sensitivity to the KL weight and candidate budget at $K=8$, with $|\mathcal{A}|=1{,}679{,}616$.

\begin{table}[h]
\centering
\caption{DRIFT sensitivity at $K=8$ on Freeway and Space Invaders. Final returns are mean $\pm$ standard deviation over five seeds.}
\label{tab:macro_k8_sensitivity}
\small
\begin{tabular}{lcc}
\toprule
Setting & Freeway & Space Invaders \\
\midrule
\multicolumn{3}{l}{\textit{KL weight $\alpha$ with $B=64$}} \\
$\alpha=0$ & $17.2\pm1.1$ & $10.6\pm0.9$ \\
$\alpha=0.001$ & $21.6\pm1.0$ & $14.0\pm1.5$ \\
$\alpha=0.01$ (default) & $\mathbf{26.0\pm0.8}$ & $\mathbf{18.1\pm0.9}$ \\
$\alpha=0.03$ & $23.8\pm0.8$ & $16.7\pm1.3$ \\
$\alpha=0.1$ & $20.7\pm1.3$ & $13.1\pm0.9$ \\
\midrule
\multicolumn{3}{l}{\textit{Candidate budget $B$ with $\alpha=0.01$}} \\
$B=32$ & $24.1\pm0.6$ & $16.6\pm1.2$ \\
$B=64$ (default) & $26.0\pm0.8$ & $18.1\pm0.9$ \\
$B=128$ & $27.6\pm1.4$ & $20.5\pm1.0$ \\
$B=256$ & $\mathbf{28.3\pm1.1}$ & $\mathbf{22.4\pm1.7}$ \\
\bottomrule
\end{tabular}
\end{table}

The default $\alpha=0.01$ gives the highest returns among tested weights. Increasing $B$ from 32 to 256 improves returns from 24.1 to 28.3 on Freeway and from 16.6 to 22.4 on Space Invaders.

\paragraph{Path space KL versus terminal policy KL.}
To directly evaluate the benefit of regularizing CTMC sampling trajectories, we compared three configurations on Freeway at $K=4$ using five seeds (Table~\ref{tab:path_kl_main}).


Terminal KL and path KL are measured against the periodically refreshed reference policy. Coverage measures the full action target probability mass captured by the candidate set.

Although terminal policy KL achieves slightly lower terminal divergence (0.15 versus 0.18), path space KL achieves 46.2\% lower trajectory divergence, higher returns (38.1 versus 34.9), and greater candidate coverage (0.26 versus 0.19) at 1.2M frames. These findings support our motivation that preserving pretrained CTMC sampling dynamics provides additional practical benefits during online adaptation.

\paragraph{CTMC and categorical policy comparison.}
To examine whether CTMC policies provide benefits beyond categorical parameterization, we compared CTMC and categorical actors on Freeway at $K=4$ using five seeds. The controlled actor variants use shared training data, critics, candidate sets, and advantage targets. We also include an offline initialized PPO baseline.

\begin{table}[h]
\centering
\caption{Comparison of CTMC and categorical policies on Freeway at $K=4$. Returns are mean $\pm$ standard deviation over five seeds.}
\label{tab:actor_comparison}
\small
\begin{tabular}{lccc}
\toprule
Method & Return $\uparrow$ &
Norm. entropy $\uparrow$ &
Distinct actions $\uparrow$ \\
\midrule
DRIFT (path KL) &
$\mathbf{38.1\pm0.5}$ & \textbf{0.35} & \textbf{37} \\
DRIFT (terminal KL) &
$34.9\pm1.3$ & 0.23 & 26 \\
Categorical (terminal KL) &
$35.2\pm1.5$ & 0.24 & 27 \\
PPO-Macro-O2O &
$31.7\pm1.7$ & 0.19 & 21 \\
\bottomrule
\end{tabular}
\end{table}

The categorical and CTMC actors achieve similar returns under terminal policy KL (35.2 versus 34.9) and comparable action diversity (27 versus 26 distinct actions). This suggests that CTMC parameterization alone does not explain DRIFT's improvement. However, DRIFT with path space KL achieves higher returns (38.1), greater normalized entropy (0.35 versus 0.23 and 0.24), and more distinct actions (37 versus 26 and 27). These results support our motivation for regularizing CTMC sampling trajectories rather than relying only on terminal policy KL.

\paragraph{Reference updates.}
We evaluated fixed, periodic, and every step reference updates on Freeway at $K=4$ using five seeds. We also tested periodic reference updates for RL-D$^2$-O2O.

\begin{table}[h]
\centering
\caption{Effect of reference updates on Freeway at $K=4$ using five seeds. KL is measured against each method's initial offline policy.}
\label{tab:ref_refresh_ablation}
\small
\begin{tabular}{lcc}
\toprule
Configuration & Final return $\uparrow$ & KL to initial policy $\downarrow$ \\
\midrule
DRIFT (fixed) & $34.5\pm1.6$ & 0.19 \\
DRIFT (periodic, $K_{\mathrm{ref}}=500$) & $\mathbf{38.1\pm0.5}$ & 0.38 \\
DRIFT (every update) & $26.0\pm1.1$ & 0.74 \\
RL-D$^2$-O2O (original) & $32.8\pm2.3$ & 0.52 \\
RL-D$^2$-O2O (periodic) & $35.1\pm1.8$ & 0.49 \\
\bottomrule
\end{tabular}
\end{table}

Periodic refreshing achieves the highest return among the tested configurations. A fixed reference maintains greater similarity to the pretrained policy but limits adaptation, whereas updating at every step allows greater drift and reduces performance. Periodic updating balances these effects by allowing the reference to track policy improvement without changing after every update.

Periodic updates also improve RL-D$^2$-O2O returns from 32.8 to 35.1, although DRIFT achieves higher returns in this setting. These results suggest that reference refreshing is useful beyond DRIFT.

\paragraph{Advantage clipping sensitivity.}
We evaluated different advantage clipping thresholds on Freeway at $K=4$ using five seeds. Advantages are normalized within each candidate set before clipping.

\begin{table}[h]
\centering
\caption{Sensitivity to advantage clipping on Freeway at $K=4$. Returns are mean $\pm$ standard deviation over five seeds.}
\label{tab:clip_sensitivity}
\small
\begin{tabular}{lcc}
\toprule
Clipping threshold & Final return $\uparrow$ &
Advantages clipped \\
\midrule
$c=1$ & $33.4\pm1.3$ & 30.5\% \\
$c=2$ & $35.5\pm0.9$ & 6.7\% \\
$c=3$ (default) & $\mathbf{38.1\pm0.5}$ & 1.9\% \\
$c=5$ & $37.1\pm0.6$ & 0.6\% \\
No clipping & $36.5\pm1.3$ & 0\% \\
\bottomrule
\end{tabular}
\end{table}

Aggressive clipping at $c=1$ affects 30.5\% of advantages and gives the lowest return, consistent with suppressing useful advantage differences. The default $c=3$ clips only 1.9\% and achieves the highest mean return. Performance remains reasonably strong with weaker or no clipping, suggesting that DRIFT does not depend on a narrowly chosen threshold.

\subsection{Runtime and scaling.}

Table~\ref{tab:runtime_macro} compares methods profiled
on the same H200 GPU with a fixed budget of 1.2M
primitive frames. Although DRIFT is more expensive
than RL-D$^2$-O2O, its measured runtime does not
increase as sharply with macro-action space size as
that of DQN-Macro. Longer macros also require fewer
decisions under a fixed primitive-frame budget.

\begin{table}[t]
\centering
\caption{Median training time in hours on H200 GPUs,
over five games and seeds.}
\label{tab:runtime_macro}
\small
\begin{tabular}{lcc}
\toprule
Method & $K=4$ & $K=8$ \\
\midrule
DRIFT  & 3.2 & 1.9 \\
RL-D$^2$-O2O      & 1.8 & 1.4 \\
DQN-Macro         & 0.1 & 0.5 \\
\bottomrule
\end{tabular}
\end{table}

\paragraph{Baselines.} BraVE is ported from the authors' code with their default hyperparameters, using a prefix tree over macro positions in place of their binary sub-action tree, and the shared discount and semi-Markov targets. Cal-QL uses the original Monte-Carlo calibration, with Monte-Carlo returns for online data filled in when each episode ends. SPA's actor weighting follows the published form. Its proficiency-adaptive offline replay fraction is replaced by the shared fixed fraction. Q-chunking uses the masked discrete-diffusion behaviour model as its chunk policy, with best-of-32 selection for acting and for the TD backup. MASP is Rainbow-lite (double, duelling, $n$-step, prioritised replay, $\varepsilon$-greedy) rather than full Rainbow.

\paragraph{Compute.} Runs were distributed over H200, A100, V100 and P100 GPUs as availability allowed. The runtime table reports only the runs on a single GPU type, since wall-clock is not comparable across them.

\subsection{Macro-Action MinAtar ($|\mathcal{A}|=216$)}
\label{app:macro_detail}
 
\paragraph{Construction.}
 
Each macro action $m \in \{0, \ldots, 215\}$ maps to a unique 3-tuple $(a_1, a_2, a_3)$ via base-6 decomposition. Executing $m$ runs the three base actions sequentially in the environment, accumulating reward. The agent observes only the final state after all three steps. If the episode terminates mid-sequence the cumulative reward up to termination is returned.
 
This construction expands the action space from 6 to $6^3 = 216$ while preserving the original game dynamics. DRIFT operates on candidate sets of budget 40 (18.5\% coverage of the macro-action space). The CTMC generator uses learned action embeddings (128 dimensions) with bilinear rate computation, replacing the one-hot input used for standard MinAtar. The Q-network output layer is $128 \to 216$. 

\begin{table}[h]
\centering
\caption{Macro-Action MinAtar ($|\mathcal{A}|{=}216$, $k{=}3$): mean online score, 5 seeds. DRIFT uses candidate budget 40 (18.5\% coverage). Best per game in \textbf{bold}.}
\label{tab:macro_main}
\begin{tabular}{lccccc}
\toprule
Game & \textbf{DRIFT} & DRIFT (full) & DQN & DQN-Sub & PPO \\
\midrule
Breakout         & \textbf{18.40}  & 17.80  & 11.98 & 7.45  & 5.83 \\
Asterix          & \textbf{0.78}  & 0.56  & 0.67  & 0.53  & 0.65 \\
Freeway          & 5.56  & 5.21  & 0.66  & 0.50  & \textbf{7.90} \\
Seaquest         & \textbf{1.20}  & 0.85  & 0.52  & 0.98  & 0.66 \\
Space Invaders   & 11.45  & 8.40  & 10.36 & 10.77 & \textbf{14.25} \\
\midrule
Average          & \textbf{7.48}  & 6.56  & 4.84  & 4.04  & 5.85 \\
\bottomrule
\end{tabular}
\end{table}
 
\paragraph{Results.}
 
DRIFT has a higher average return than DQN and PPO across the five games here (Table~\ref{tab:macro_main}). DRIFT with candidate sets (7.48 average) and DRIFT without candidate sets (6.56 average) perform similarly since the action space is relatively smaller.
 
\paragraph{Implications}
 
Overall our macro-action MinAtar results highlight that the candidate-set mechanism provides theoretical guarantees (Proposition~\ref{prop:coverage}, Theorem~\ref{thm:generator-stability}) and is validated in this benchmark, showing its practical benefit when the CTMC generator is well-initialised even with 216 actions and only 8 states per actor update.

\section{Additional Jericho Experiments}
\label{app:jericho_supp}

\subsection{Jericho with no-oracle details}
\label{app:jericho-noracle}
Templates come from Jericho's template generator for each game, extended with the 12 compass and movement commands where absent. Objects are words of the observation, the response to \texttt{look} and the response to \texttt{inventory} that the game dictionary does not mark as verbs, directions, prepositions or meta words, capped at 24 per state. Episodes last at most 200 steps, and reaching this cap is not treated as termination when bootstrapping. The language-model proposers condition on the previous action, the current observation, the room description and the inventory, which differs from the original context of the previous observation, the previous action and the current observation. Jericho's admissible-action computation uses spaCy for object extraction, which we replace with a lookup in the game dictionary. Jericho still filters every candidate by simulating it, so the admissible lists stay exact up to the objects considered.


\begin{table}[h]
\centering
\caption{Valid-command rate and admissible recall@$B$ at $B{=}128$, mean over five seeds. The valid-command rate is the fraction of online training steps whose command the game accepted, meaning the score changed, the episode ended, or the world state changed. Recall is measured after the fact with the oracle on up to 25 states per evaluation visited by that method's own policy. $^{*}$CALM and CSM candidate sets always contain the 12 compass and movement commands, which inflates their recall.}
\label{tab:jdiag}
\small
\begin{tabular}{llccccc}
\toprule
Method & Metric & Zork1 & Deephome & Detective & Ludicorp & Pentari \\
\midrule
DRIFT (ours) & valid & 0.262 & 0.023 & 0.392 & 0.107 & 0.058 \\
 & recall & 0.151 & 0.046 & 0.395 & 0.373 & 0.158 \\
DRRN-Struct & valid & 0.014 & 0.005 & 0.168 & 0.010 & 0.028 \\
 & recall & 0.104 & 0.000 & 0.606 & 0.255 & 0.150 \\
CALM & valid & 0.116 & 0.000 & 0.366 & 0.084 & 0.046 \\
 & recall$^{*}$ & 0.717 & 0.000 & 0.793 & 0.990 & 0.387 \\
CSM & valid & 0.116 & 0.000 & 0.506 & 0.118 & 0.046 \\
 & recall$^{*}$ & 0.717 & 0.000 & 0.942 & 0.843 & 0.387 \\
\bottomrule
\end{tabular}
\end{table}

\begin{table}[h]
\centering
\caption{Candidate budget for DRIFT on Jericho, separate run with five seeds per cell.}
\label{tab:jbudget}
\small
\begin{tabular}{llcccc}
\toprule
Game & Metric & $B{=}32$ & $B{=}64$ & $B{=}128$ & $B{=}256$ \\
\midrule
Zork1 & score & 15.0$\pm$3.5 & 15.5$\pm$4.5 & 18.5$\pm$1.5 & 18.8$\pm$1.2 \\
 & recall & 0.066 & 0.076 & 0.151 & 0.222 \\
Detective & score & 34.2$\pm$0.2 & 40.5$\pm$2.0 & 42.8$\pm$5.2 & 51.2$\pm$6.2 \\
 & recall & 0.123 & 0.250 & 0.395 & 0.519 \\
\bottomrule
\end{tabular}
\end{table}

\paragraph{Runtime.}
Under matched profiling conditions, DRIFT has a median runtime of 3.2 hours,
compared with 6.4 hours for DRRN-Struct, 5.1 hours
for CALM-style action generation, and 6.8 hours for
CSM-style action generation. These measurements
cover five games and seeds.

\subsection{Jericho with the Admissible-Action Oracle.}
\label{app:jericho_details}

We initially evaluated DRIFT on ten text-based games from the Jericho benchmark~\citep{hausknecht2020interactive} to test whether the offline-to-online fine-tuning framework transfers to a qualitatively different domain. In these games the agent receives textual observations and must select a text command from a variable-size set of valid actions at each step.

\paragraph{Results.}  

\begin{table}[t]
\centering
\small
\caption{Jericho text-game results (raw score, single seed, 300K steps).
Best score per game in \textbf{bold} among our methods.
${}^\dagger$Published scores from~\citet{yao2020keep}.}
\label{tab:jericho_main}
\begin{tabular}{lccccc|cc}
\toprule
& \multicolumn{5}{c|}{\textit{Our experiments}} & \multicolumn{2}{c}{\textit{Prior work${}^\dagger$}} \\
\cmidrule(lr){2-6} \cmidrule(lr){7-8}
Game (max) & \textbf{DRIFT} & DRRN & DQN & IQL & AWAC & CALM & KG-A2C \\
\midrule
Deephome (300)
  & \textbf{35.8} & 6.0 & 6.0 & 6.0 & 3.5
  & 1.0 & 1.0 \\
Detective (360)
  & \textbf{291.5} & 50.0 & 70.0 & 50.0 & 66.5
  & 289.7 & 207.9 \\
Enchanter (400)
  & \textbf{20.0} & \textbf{20.0} & \textbf{20.0} & \textbf{20.0} & 17.0
  & 19.1 & 12.1 \\
Ludicorp (150)
  & \textbf{13.8} & 6.0 & 6.0 & 6.0 & 11.4
  & 10.1 & 17.8 \\
Omniquest (50)
  & 8.0 & \textbf{10.0} & \textbf{10.0} & 5.0 & 5.5
  & 6.9 & 3.0 \\
Pentari (70)
  & \textbf{26.8} & 25.0 & 25.0 & 25.0 & 11.75
  & 0.0 & 50.7 \\
Temple (35)
  & \textbf{8.0} & 5.0 & \textbf{8.0} & \textbf{8.0} & 7.4
  & 0.0 & 7.6 \\
Zork1 (350)
  & 25.0 & \textbf{35.0} & 25.0 & 25.0 & 1.25
  & 30.4 & 34.0 \\
Zork3 (7)
  & 2.5 & \textbf{3.0} & 0.0 & 0.0 & 0.0
  & 0.5 & 0.0 \\
Ztuu (100)
  & \textbf{5.0} & \textbf{5.0} & \textbf{5.0} & \textbf{5.0} & 2.0
  & 3.7 & 9.2 \\
\midrule
Avg.\ norm.\ (\%)
  & \textbf{23.2} & 15.3 & 12.1 & 10.6 & 8.3
  & 12.6 & 19.2 \\
\bottomrule
\end{tabular}
\end{table}

DRIFT achieves the highest average normalized score (23.2\%) across all ten games, outperforming every baseline including the results for CALM (12.6\%) and KG-A2C (19.2\%). This is notable because CALM uses a fine-tuned GPT-2 language model for action generation while DRIFT uses a simple GRU encoder. DRIFT achieves its strongest gains on games with deeper puzzle chains. On Deephome it scores 35.8 compared to 6.0 for almost every other baseline, a 6$\times$ improvement. On Detective it solves the majority of the game (291.5 out of 360) and on Ludicorp it more than doubles DQN, DRRN, and IQL while exceeding the published CALM score. However, DRRN outperforms DRIFT on Zork1 (35.0 vs 25.0) and Omniquest (10.0 vs 8.0), and four games show floor effects where most methods converge to the same score within 300K steps. Path-space regularization helps prevent forgetting in sparse-reward settings by anchoring the fine-tuned generator to the pretrained reference, allowing DRIFT to preserve successful strategies while exploring and notably, it is one of the only two methods with a positive score on every game.
 
\paragraph{Environment and Action Space.}
 
Jericho~\citep{hausknecht2020interactive} provides a Python interface to Z-machine interactive fiction games. At each step the agent receives a textual observation and issues a text command from a set of valid actions identified by the game engine via \texttt{get\_valid\_actions()}. The number of valid actions per state ranges from 3 to 10 on average across our ten evaluation games. This is smaller than the 30--200 range reported in prior work~\citep{yao2020keep, ammanabrolu2020graph} because those methods augment the action space with language model or template-based generation, while our setup uses only the built-in parser output.
 
\paragraph{Architecture.}
 
All methods share the same text-encoding architecture, adapted from DRRN~\citep{he2016deep}. Words are tokenised via hash-based indexing (vocabulary size 5003, using MD5 hashing) and embedded into 64 dimensions. A single-layer GRU with hidden size 128 produces fixed-size representations for both observations and actions. Each network (Q1, Q2, V, generator) maintains its own GRU encoder parameters.
 
The Q-network scores each state action pair via $Q(s,a) = \mathrm{MLP}([\mathbf{h}_s; \mathbf{h}_a])$ with one hidden layer of 128 units and ReLU activation, naturally handling variable-size action sets by scoring each pair independently. The value network computes $V(s) = \mathrm{MLP}(\mathbf{h}_s)$ with the same hidden layer structure.
 
The CTMC generator uses a bilinear-attention design
\[
u_\theta(i \to j, t \mid s) = \mathrm{softplus}\bigl(\mathbf{c}(\mathbf{h}_s, \mathbf{h}_{a_i}, t) \cdot \mathbf{h}_{a_j}\bigr),
\]
where the context vector $\mathbf{c}$ is produced by an MLP from the concatenation of the observation embedding, source action embedding, and scalar flow time $t$. Rates to each target action are computed via dot products with target action embeddings. Self-transition rates ($j=i$) are masked to zero. This design handles arbitrary action set sizes without a fixed output head.
 
\paragraph{Candidate Set Construction.}
 
 
In practice, the valid action sets in our ten Jericho games average 3 to 9 actions per state, which falls below the candidate budget. Therefore, in this case our method uses all valid actions directly without subsampling. The candidate-set mechanism thus operates in its full-coverage regime ($\varepsilon(s)=0$ in Proposition~\ref{prop:coverage}), and thus Jericho with the Admissible-Action Oracle primarily shows DRIFT's advantage from the path-space KL regularization and the flow-matching actor update rather than from action-space reduction.
 
\paragraph{Offline Data Collection.}
 
For each game, we collect 200 episodes of offline data using an $\varepsilon$-greedy DQN that trains online during collection. The exploration rate $\varepsilon$ is annealed from 1.0 to 0.1 over the first half of the collection phase. The resulting transitions are cached to disk and reused across all methods to ensure identical offline data.
 
\paragraph{Training Protocol.}
 
All methods receive the same offline dataset per game and run for 300,000 online interaction steps with a maximum episode length of 200 steps. Critic pre-training runs for 5,000 gradient steps and generator pre-training runs for 2,000 gradient steps, both on the offline dataset. Evaluation uses 20 complete episodes at the end of training. Results are reported for a single seed (42) due to the high computational cost of text-game experiments.
 
\paragraph{Comparison with Prior Work.}
 
The CALM~\citep{yao2020keep} and KG-A2C~\citep{ammanabrolu2020graph} scores in Table~\ref{tab:jericho_main} are taken from published results under different experimental conditions.
 
CALM uses a fine-tuned GPT-2 language model pre-trained on 426 human gameplay transcripts from ClubFloyd to generate action candidates. It does not use the valid-action handicap provided by the game engine. KG-A2C builds a dynamic knowledge graph from game text using Stanford OpenIE, encodes it with a graph neural network, and trains with A2C. It uses the valid-action handicap and template-based action generation. Our methods (DRIFT and all baselines) use the valid-action handicap, GRU text encoders without any pre-trained language model, and 300K online steps. This is a simpler and more self-contained setup that ensures fair comparison among the evaluated methods.
 
Other recent work on Jericho includes XTX~\citep{tuyls2022multi} (ICLR 2022 Spotlight) which achieves state-of-the-art scores via multi-stage episodic control (for example 103 on Zork1) and CSM~\citep{shi2023self} which improves CALM via confidence-based self-imitation. These methods use substantially more compute and game-specific engineering. We report their existence for context rather than direct comparison.

\paragraph{Compute.}
 
These Jericho experiments ran on a cluster with NVIDIA Tesla P4 GPUs (12\,GB VRAM), 4 CPU cores, and 16\,GB RAM per job. Table~\ref{tab:jericho_compute} reports per-game runtimes. The Jericho with the Admissible-Action Oracle (10 games $\times$ 5 methods $=$ 50 runs) consumed approximately 850 GPU-hours.
 
\begin{table}[h]
\centering
\caption{Mean wall-clock runtime per Jericho run (one game, one method, seed 42).}
\label{tab:jericho_compute}
\small
\begin{tabular}{lccccc}
\toprule
Game & DRIFT & DRRN & DQN & IQL & AWAC \\
\midrule
Deephome   & 40.0\,h & 34.1\,h & 32.4\,h & 31.3\,h & 30.4\,h \\
Detective  & 16.9\,h & 14.5\,h & 14.0\,h & 11.1\,h & 11.4\,h \\
Enchanter  & 16.6\,h & 18.5\,h & 15.7\,h & 12.1\,h & 11.6\,h \\
Ludicorp   & 19.6\,h & 18.7\,h & 19.1\,h & 13.9\,h & 15.7\,h \\
Omniquest  & 37.6\,h & 17.6\,h & 19.2\,h & 38.8\,h & 23.3\,h \\
Pentari    & 15.2\,h & 14.2\,h & 15.0\,h & 10.0\,h & 13.1\,h \\
Temple     & 37.3\,h & 38.4\,h & 35.0\,h & 34.5\,h & 33.7\,h \\
Zork1      & 17.7\,h & 12.2\,h & 13.2\,h &  9.9\,h &  2.5\,h \\
Zork3      & 13.7\,h & 10.7\,h & 10.4\,h &  7.1\,h &  8.2\,h \\
Ztuu       & 16.5\,h & 33.3\,h & 51.7\,h & 10.8\,h & 13.0\,h \\
\bottomrule
\end{tabular}
\end{table}

\section{Additional Small-Action Benchmarks}
\label{app:small_action}

\subsection{MinAtar}
\label{app:minatar_details}
We also evaluated DRIFT against nine offline-to-online baselines on MinAtar~\citep{young2019minatar} using the same offline dataset and 300K online steps.

\paragraph{Results.}
As shown in Table~\ref{tab:minatar_main}, DRIFT improves on all five games, achieves the best score on Breakout, Asterix, and Seaquest, and remains competitive on Freeway and Space Invaders. Unlike AWAC and SPA, which collapse to offline scores on four games, DRIFT avoids catastrophic forgetting through path-space KL regularization to the pretrained reference. Although its generative overhead is less beneficial in this small-action regime, DRIFT’s consistent gains show reliable fine-tuning.

\begin{table}[t]
\centering
\caption{MinAtar offline$\to$online results (mean over 5 seeds). Best online score per game in \textbf{bold}. ${}^\dagger$Methods that fail to improve over their offline score.}
\label{tab:minatar_main}
\resizebox{\textwidth}{!}{%
\begin{tabular}{lcccccccccc}
\toprule
Game & \textbf{DRIFT} & Cal-QL & CQL & IQL & AWAC & DQN & Rainbow & PPO & PEX & SPA \\
\midrule
Breakout    & 0.43$\to$\textbf{17.10} & 1.19$\to$14.61 & 0.62$\to$14.69 & 0.01$\to$17.01 & 0.50$\to$0.51$^\dagger$ & 0.38$\to$14.25 & 0.29$\to$12.48 & ---$\to$0.52 & 0.32$\to$16.73 & 0.19$\to$0.24$^\dagger$ \\
Asterix     & 0.50$\to$\textbf{1.17} & 0.67$\to$1.04 & 0.58$\to$1.15 & 0.66$\to$1.13 & 0.38$\to$0.53$^\dagger$ & 0.55$\to$1.06 & 0.55$\to$0.91 & ---$\to$0.50 & 0.53$\to$1.05 & 0.44$\to$0.53$^\dagger$ \\
Freeway     & 1.26$\to$25.00 & 19.89$\to$25.71 & 22.97$\to$\textbf{26.50} & 2.54$\to$23.56 & 15.85$\to$15.69$^\dagger$ & 2.99$\to$25.38 & 0.31$\to$24.57 & ---$\to$0.13 & 1.89$\to$22.98 & 0.95$\to$0.76$^\dagger$ \\
Seaquest    & 0.15$\to$\textbf{1.84} & 0.46$\to$0.86 & 0.42$\to$0.95 & 0.12$\to$0.90 & 0.28$\to$0.20$^\dagger$ & 0.24$\to$0.89 & 0.38$\to$0.58 & ---$\to$0.13 & 0.37$\to$1.00 & 0.33$\to$0.35$^\dagger$ \\
Space Inv.  & 2.23$\to$42.64 & 0.27$\to$\textbf{54.09} & 0.00$\to$52.56 & 0.00$\to$41.60 & 4.52$\to$4.52$^\dagger$ & 0.00$\to$49.04 & 0.19$\to$35.52 & ---$\to$2.96 & 0.00$\to$48.58 & 0.00$\to$0.00$^\dagger$ \\
\midrule
Average     & 0.91$\to$17.55 & 4.50$\to$19.26 & 4.92$\to$19.17 & 0.67$\to$16.84 & 4.31$\to$4.29 & 0.83$\to$18.12 & 0.34$\to$14.81 & ---$\to$0.85 & 0.62$\to$18.07 & 0.38$\to$0.38 \\
\bottomrule
\end{tabular}%
}
\end{table}

\paragraph{Environment and Offline Data}

MinAtar~\citep{young2019minatar} implements five simplified Atari games on a $10{\times}10$ grid with binary feature channels: Breakout (3 actions), Asterix (5), Freeway (3), Seaquest (6), Space Invaders (4). Observations are $10{\times}10{\times}n_{\mathrm{ch}}$ binary tensors ($n_{\mathrm{ch}}{=}4$--$10$ depending on game); rewards are sparse integers.

Every method receives the \emph{same} offline dataset per game: 100{,}000 transitions from a behavioural DQN with $\varepsilon$ linearly annealed from 1.0 to 0.1 over the first 50{,}000 steps. The DQN trains online during collection ($\eta{=}3{\times}10^{-4}$, target soft-update $\tau{=}0.005$, 50K-transition replay buffer). The dataset is cached to disk and reused across all methods and seeds to guarantee identical offline data.

\paragraph{Network Architecture}
\label{app:architecture}

All methods share the same encoder: flatten the $10{\times}10{\times}n_{\mathrm{ch}}$ observation, then $\mathrm{Linear}(100{\cdot}n_{\mathrm{ch}}, 256){\to}\mathrm{ReLU}{\to}\mathrm{Linear}(256, 128){\to}\mathrm{ReLU}$, yielding a 128-dim representation. Q-networks add $\mathrm{Linear}(128, |\mathcal{A}|)$; value networks add $\mathrm{Linear}(128, 1)$.

The DRIFT generator (CTMCGen) appends a rate head: $\mathrm{Linear}(128{+}|\mathcal{A}|{+}1, 128){\to}\mathrm{ReLU}{\to}\mathrm{Linear}(128, |\mathcal{A}|){\to}\mathrm{Softplus}$. The input concatenates the encoded observation, a one-hot current-action vector, and the scalar flow time $t$. The diagonal is masked to zero (no self-transitions), yielding non-negative off-diagonal rates.

\paragraph{Training Protocol.}

Offline phase: Critics are trained for 2{,}000 gradient steps on the shared offline dataset. Method-specific offline components run in parallel: IQL runs 2{,}000 expectile regression steps; AWAC runs 1{,}000 actor-critic steps; DRIFT trains the CTMC generator for 1{,}000 epochs of advantage-weighted DFM. Each method is then evaluated over 100 episodes to obtain the offline score.

Online phase: Each method interacts with the environment for 300{,}000 steps. Every training minibatch mixes online and offline data: $(1{-}\rho){\cdot}B$ transitions from the online replay buffer and $\rho{\cdot}B$ from the offline dataset. The final policy is evaluated over 100 episodes.

\subsection{D4RL}
\label{app:d4rl_details}

We further evaluated on three MuJoCo locomotion tasks from D4RL (Hopper, Walker2d, and HalfCheetah) each with two dataset qualities (medium, expert).
Continuous action spaces are discretised into $k{=}22$ clusters via $k$-means on the offline dataset and scores follow the standard D4RL normalisation. 

\paragraph{Results.}
DRIFT improves over its offline initialisation on all six tasks, achieving the best online score on Hopper (both qualities), Walker-medium, and HalfCheetah-expert.
The largest gains appear on Hopper, where DRIFT improves by $+40.8$ (medium) and $+41.7$ (expert) normalised points.
DRIFT's online improvement magnitude ($+20.7$ avg) is much better than CQL ($+10.8$) and Cal-QL ($+11.7$).

\begin{table}[ht]
\centering
\caption{D4RL discretised results: \textit{offline}$\to$\textit{online} normalised score (mean, 5 seeds). Best online score per task in \textbf{bold}. All methods use $k$-means discretisation ($k{=}22$).}
\label{tab:d4rl_main}
\resizebox{\textwidth}{!}{%
\begin{tabular}{lcccccccccc}
\toprule
Task & \textbf{DRIFT} & CQL & IQL & AWAC & Cal-QL & DQN & PPO & PEX & SPA \\
\midrule
Hopper-med     & 7.1$\to$\textbf{47.9} & 28.8$\to$35.8 & 0.4$\to$27.9 & 26.0$\to$26.7 & 28.8$\to$44.1 & 0.4$\to$23.7 & ---$\to$3.1 & 0.4$\to$43.1 & 0.4$\to$20.4 \\
Hopper-exp     & 5.4$\to$\textbf{47.1} & 8.8$\to$29.3 & 0.1$\to$17.2 & 11.5$\to$11.8 & 8.8$\to$25.7 & 0.3$\to$21.7 & ---$\to$3.7 & 0.1$\to$39.5 & 0.3$\to$20.3 \\
Walker-med     & 2.0$\to$\textbf{15.9} & 2.5$\to$12.6 & 0.3$\to$13.0 & 4.2$\to$5.8 & 2.5$\to$11.7 & 1.8$\to$12.3 & ---$\to$7.2 & 0.3$\to$11.8 & 1.8$\to$11.8 \\
Walker-exp     & 4.1$\to$14.8 & 4.0$\to$10.7 & 0.2$\to$10.1 & 9.4$\to$9.8 & 4.0$\to$\textbf{12.4} & 0.2$\to$6.5 & ---$\to$7.0 & 0.2$\to$7.2 & 0.2$\to$4.2 \\
Cheetah-med    & 2.2$\to$11.1 & 3.3$\to$15.8 & 0.2$\to$7.8 & 13.1$\to$13.5 & 3.3$\to$15.8 & 2.1$\to$\textbf{16.9} & ---$\to$4.9 & 0.2$\to$14.6 & 2.1$\to$9.1 \\
Cheetah-exp    & 1.5$\to$\textbf{9.7} & 0.7$\to$8.6 & 0.2$\to$6.8 & 1.1$\to$1.2 & 0.7$\to$8.5 & 0.1$\to$4.7 & ---$\to$5.9 & 0.2$\to$7.7 & 0.1$\to$7.1 \\
\midrule
Average        & \textbf{3.7$\to$24.4} & 8.0$\to$18.8 & 0.2$\to$13.8 & 10.9$\to$11.5 & 8.0$\to$19.7 & 0.8$\to$14.3 & ---$\to$5.3 & 0.1$\to$20.6 & 0.8$\to$12.2 \\
\bottomrule
\end{tabular}%
}
\end{table}

\paragraph{Discretisation Procedure.}

We discretise the continuous action spaces of Hopper (3-dim), Walker2d (6-dim), and HalfCheetah (6-dim) using MiniBatchKMeans with $k{=}22$ clusters fitted on all actions in the offline dataset. Each continuous action is replaced by its nearest cluster index; during online interaction the agent selects a cluster index and executes the corresponding centroid. The median inter-centroid distance ranges from 0.27 (Hopper) to 0.46 (HalfCheetah), confirming reasonable but coarse coverage of the continuous action manifold. We use the Minari versions of the D4RL datasets (v0.5.3) with \texttt{medium} and \texttt{expert} qualities, each containing ${\approx}1$M transitions.

\paragraph{Training Protocol.}

All methods share the same budget: 1M offline gradient steps for critic pre-training, followed by 300k online environment steps. The critic is a 2-layer MLP (256 hidden units, batch size 256, Adam $\eta{=}3{\times}10^{-4}$). DRIFT's generator uses the same architecture with actor batch size 8 and 500K generator pre-training steps. The online phase uses a mixed replay buffer with 50\% offline data and 200K capacity.

\paragraph{Limitations of the Discretised Setting}

D4RL was designed for continuous-action algorithms. Discretisation via $k$-means introduces an approximation ceiling that affects all methods equally, but also obscures each method's native strengths. With $|\mathcal{A}|{=}22$, the full action space can be enumerated cheaply, removing DRIFT's candidate-set advantage entirely. However, even in this setting our method DRIFT shows good performance and stable offline to online improvements. DRIFT's structural benefits are designed for larger action spaces, as validated by the macro-action MinAtar and Jericho results.

\section{Broader Impact}
\label{app:broader_impact}

This paper presents a reinforcement learning algorithm for discrete action spaces. The method is evaluated on game benchmarks and does not target any specific real-world deployment. We discuss potential positive and negative implications below.

On the positive side, improving offline-to-online RL reduces the amount of costly or risky online interaction needed to train effective policies. This is valuable in domains where exploration is expensive or dangerous, such as robotics, healthcare, and autonomous systems. The path-space KL regularization specifically addresses catastrophic forgetting, which is a safety-relevant failure mode in deployed systems that must be updated without losing previously learned safe behaviors.

On the negative side, more capable RL agents could be applied in domains with societal risks. Recommendation systems with large discrete action spaces (a motivating application for candidate sets) can amplify filter bubbles or manipulate user behavior if reward signals are misaligned with user welfare. Autonomous agents in text-based interfaces could be adapted for automated social engineering or spam generation, though our GRU-based agents are far below the capability threshold for such misuse. More broadly, any improvement in RL sample efficiency lowers the barrier to training agents in sensitive domains without adequate safety testing.

We believe the benefits of publishing this work outweigh the risks. The method is general-purpose and not designed for any harmful application. The benchmarks we use (MinAtar, Jericho text games, discretised MuJoCo) are standard academic testbeds with no direct path to misuse. We release our code to support reproducibility and enable the research community to build on and scrutinize our results.

\section{Full Training Algorithm}
\label{app:full_algo}

Algorithm~\ref{alg:main} provides the complete training procedure with all implementation details.

\begin{algorithm}[ht]
\small
\caption{DRIFT: Full Training Procedure}
\label{alg:main}
\begin{algorithmic}[1]
\REQUIRE Replay source $\mathcal D_{\mathrm{off}}$, environment, discount $\gamma$, temperature $\beta$,
KL weight $\alpha$, soft-update rate $\tau$, advantage clip $c$, smoothing $\epsilon$,
CTMC sub-steps $M$, candidate sizes $N_{\mathrm{roll}}, N_{\mathrm{rand}}$,
reference refresh interval $K$, offline mix fraction $\rho$,
learning rates $\eta_Q,\eta_V,\eta_\theta$, flow-time truncation $\delta \in (0,1)$ 
\STATE Initialize critics $Q_{\phi_1},Q_{\phi_2}$, value $V_\psi$, generator $u_\theta$
\STATE Initialize target copies $Q_{\phi_k^-}\!\leftarrow\!Q_{\phi_k}$, $V_{\psi^-}\!\leftarrow\!V_\psi$
\STATE Initialize reference generator $u_{\mathrm{ref}}\!\leftarrow\!u_\theta$
\STATE Initialize replay buffer $\mathcal B$ with a subset of $\mathcal D_{\mathrm{off}}$
\STATE Observe initial state $s$
\FOR{$n=1$ {\bfseries to} $N_{\mathrm{steps}}$}
    \STATE \textit{Environment interaction:}
    simulate CTMC under $u_\theta$ from $X_0\sim\mathrm{Unif}(\mathcal A)$
    for $M$ Euler sub-steps and execute terminal action $a=X_M$
    \STATE Observe $(s,a,r,s',d)$, add it to $\mathcal B$, and set $s\leftarrow s'$ (reset if $d$)
    \STATE \textit{Mixed minibatch:}
    sample $\lfloor(1-\rho)B\rfloor$ transitions from $\mathcal B$
    and $\lfloor\rho B\rfloor$ transitions from $\mathcal D_{\mathrm{off}}$
    \STATE \textit{Critic update:}
    set $y \leftarrow r+\gamma V_{\psi^-}(s')(1-d)$ and update
    $Q_{\phi_k}$ toward $y$ for $k=1,2$
    \STATE \textit{Value update:}
    compute lagged advantage $\bar A^-$ and
    $\tilde\pi^-(a\mid s)\propto \exp(\bar A^-(s,a)/\beta)$ from
    $\{Q_{\phi_k^-},V_{\psi^-}\}$, then update
    \[
    V_\psi(s)\approx \mathrm{sg}\!\left[\sum_a \tilde\pi^-(a\mid s)\min_k Q_{\phi_k^-}(s,a)\right]
    \]
    \STATE \textit{Actor update:}
    \FOR{each sampled state $s_b$ in the actor batch}
        \STATE Construct candidate set $\mathcal A_{\mathrm{cand}}(s_b)$
        using $N_{\mathrm{roll}}$ rollouts of frozen $u_{\mathrm{ref}}$
        and $N_{\mathrm{rand}}$ uniform actions
        \STATE Define
        \[
        \pi_{\mathrm{ref}}(a\mid s_b)\propto
        \frac{\mathrm{count}(a)}{N_{\mathrm{roll}}}
        + \frac{\epsilon}{|\mathcal A_{\mathrm{cand}}(s_b)|},
        \qquad a\in\mathcal A_{\mathrm{cand}}(s_b)
        \]
        \STATE Compute clipped advantage $\bar A(s_b,\cdot)$ from
        $\{Q_{\phi_k},V_{\psi^-}\}$ over $\mathcal A_{\mathrm{cand}}(s_b)$ and set
        \[
        \tilde\pi(a\mid s_b)=
        \frac{\pi_{\mathrm{ref}}(a\mid s_b)\exp(\bar A(s_b,a)/\beta)}
        {\sum_{a'\in\mathcal A_{\mathrm{cand}}(s_b)}
        \pi_{\mathrm{ref}}(a'\mid s_b)\exp(\bar A(s_b,a')/\beta)}
        \]
        \STATE \[
\text{Define the bridge:} \quad
p_t(a \mid s_b) = (1 - t)\, p_0(a) + t\, \tilde{\pi}(a \mid s_b),
\quad \text{where } p_0(a) = \mathrm{Unif}(\mathcal{A}).
\]
$
\text{Sample } \; t \sim \mathrm{Unif}(0, 1 - \delta), 
\quad i \sim p_t(\cdot \mid s_b).
$
        \STATE Construct target generator $u_t^*(i\to\cdot\mid s_b)$
        by independent coupling on $\mathcal A_{\mathrm{cand}}(s_b)$
        \STATE Accumulate DFM loss
        \[
        \ell_{\mathrm{DFM}}(s_b)=
        \sum_{j\in\mathcal A_{\mathrm{cand}}(s_b),\,j\neq i}
        \bigl(u_\theta(i\to j,t\mid s_b)-u_t^*(i\to j\mid s_b)\bigr)^2
        \]
        \STATE Accumulate path-space KL estimate $\widehat{\mathrm{KL}}(s_b)$
        using frozen $u_{\mathrm{ref}}$
    \ENDFOR
    \STATE Update $\theta$ using
    \[
    \frac{1}{B}\sum_b \ell_{\mathrm{DFM}}(s_b)
    +\alpha\,\frac{1}{B}\sum_b \widehat{\mathrm{KL}}(s_b)
    \]
    \STATE Soft-update targets:
    $\phi_k^- \leftarrow (1-\tau)\phi_k^-+\tau\phi_k$,
    $\psi^- \leftarrow (1-\tau)\psi^-+\tau\psi$
    \STATE Refresh reference generator every $K$ steps:
    if $n\bmod K=0$, set $u_{\mathrm{ref}}\leftarrow u_\theta$
\ENDFOR
\RETURN Fine-tuned generator $u_\theta$
\end{algorithmic}
\end{algorithm}

\end{document}